\documentclass[accepted]{uai2023_no_decorations}
\usepackage[american]{babel}
\usepackage{natbib}
\usepackage{longtable}
\usepackage{mathtools}
\usepackage{amssymb}
\usepackage{hyperref}
\usepackage{amsthm,mathtools}
\usepackage{bm}
\usepackage{booktabs}
\usepackage{xcolor}
\usepackage{hyperref}
\usepackage{enumitem}
\usepackage{subcaption}

\newcommand*{\tran}{^{\mkern-1.5mu\mathsf{T}}}
\newcommand{\E}{\mathbb{E}}

\newcommand{\Q}{\mathbb{Q}}   
\newcommand{\R}{\mathbb{R}}   
\newcommand{\X}{\mathcal{X}}  
\newcommand{\Y}{\mathcal{Y}}  
\newcommand{\ip}[2]{\left\langle{#1}\right\rangle_{#2}} 
\newcommand{\kme}[1]{\mu_{k_{#1}}(\P{#1})} 
\newcommand{\norm}[2]{\left\|{#1}\right\|_{#2}}
\newcommand{\nys}{\text{Nys}} 
\newcommand{\opnorm}[1]{\left\|{#1}\right\|_{\mathrm{op}}}
\newcommand{\prodmarginals}{\otimes_{m=1}^M \P_ m} 
\newcommand{\tb}{\textbf}    
\newcommand{\tphs}{\otimes_{m=1}^M \H_{k_m}}

\renewcommand{\H}{\mathcal{H}} 
\renewcommand{\O}{\mathcal{O}} 
\renewcommand{\P}{\mathbb{P}} 
\renewcommand{\b}{\mathbf}    
\renewcommand{\d}{\mathrm{d}} 
\newcommand{\kmeP}{\mu_{k}\left(\tilde \P_{n'}\right)} 
\newcommand{\kmePm}{\mu_{k_m}\left(\tilde \P_{m,n'}\right)}

\DeclareMathOperator{\trace}{tr}

\DeclareMathOperator{\hadamard}{\circ}
\DeclareMathOperator{\HSIC}{HSIC}

\newcommand{\Psamp}{\mathbb{\hat P}}

\newcommand{\Span}{\mathrm{span}} 
\usepackage{mathabx}

\newtheorem{lemma}{Lemma}[section]
\newtheorem{theorem}{Theorem}[section]
\newtheorem{proposition}{Proposition}[section]
\newtheorem{remark}{Remark}

\title{Nyström $M$-Hilbert-Schmidt Independence Criterion}

\author[1]{Florian~Kalinke}
\author[2]{Zoltán~Szabó}
\affil[1]{Institute for Program Structures and Data Organization, Karlsruhe Institute of Technology, Karlsruhe, Germany}
\affil[2]{Department of Statistics, London School of Economics, London, UK}
  
\begin{document}

\maketitle

\begin{abstract}
Kernel techniques are among the most popular and powerful approaches of data science. Among the key features that make kernels ubiquitous are (i) the number of domains they have been designed for, (ii) the Hilbert structure of the function class associated to kernels facilitating their statistical analysis, and (iii) their ability to represent probability distributions without loss of information. These properties give rise to the immense success of Hilbert-Schmidt independence criterion (HSIC) which is able to capture joint independence of random variables under mild conditions, and permits closed-form estimators with quadratic computational complexity (w.r.t.\ the sample size). In order to alleviate the quadratic computational bottleneck in large-scale applications, multiple HSIC approximations have been proposed, however these estimators are restricted to $M=2$ random variables, do not extend naturally to the $M\ge 2$ case, and lack theoretical guarantees. In this work, we propose an alternative Nyström-based HSIC estimator which handles the $M\ge 2$ case, prove its consistency, and  demonstrate its applicability in multiple contexts, including synthetic examples, dependency testing of media annotations, and causal discovery.
\end{abstract}

\section{Introduction}
Kernels methods \citep{aronszajn50theory} have been on the forefront of data science for more than $20$ years \citep{scholkopf02learning,steinwart08support}, and they underpin some of the most powerful and principled machine learning techniques currently known. The key idea of kernels is to map the data into a (possibly infinite-dimensional) feature space in which one computes the inner product implicitly by means of a symmetric, positive definite function, the so-called kernel function.

Kernel functions have been designed for 
strings \citep{watkins99dynamic,lodhi02text} or more generally for sequences \citep{kiraly19kernel}, sets \citep{haussler99convolution, gartner02multi}, rankings \citep{jiao16kendall}, fuzzy domains \citep{guevara17cross} and graphs \citep{borgwardt20graph}, which renders them broadly applicable.
Their extension to the space of probability measures \citep{berlinet04reproducing,smola07hilbert} allows to represent distributions in a reproducing kernel Hilbert space (RKHS) by the so-called mean embedding. 
Such embeddings form the main building block of maximum mean discrepancy (MMD;  \citet{smola07hilbert,gretton12kernel}), which quantifies the discrepancy of two distributions as the RKHS distance of their respective mean embeddings. 
MMD is (i) a semi-metric on probability measures, (ii) a metric iff.\ the kernel is characteristic \citep{fukumizu08kernel,sriperumbudur10hilbert}, (iii) an instance of integral probability metrics  (IPM; \citet{muller97integral,zolotarev83probability}) when the underlying function class in the IPM is chosen to be the unit ball in an RKHS. 

Measuring the discrepancy of a joint distribution to the product of its marginals by MMD gives rise to the Hilbert-Schmidt independence criterion (HSIC; \citet{gretton05measuring}). 
HSIC was shown to be equivalent \citep{sejdinovic13equivalence} to distance covariance \citep{szekely07measuring,szekely09brownian,lyons13distance}; \citet{sheng23distance} have recently proved a similar result for the conditional case. 
HSIC is known to capture the independence of $M=2$ random variables with characteristic $(k_m)_{m=1}^2$ kernels (on the respective domains) as proved by \citet{lyons13distance}; for more than two components ($M>2$; \citet{quadrianto09kernelized,sejdinovic13kernel,pfister18kernel}) universality \citep{steinwart01influence, micchelli06universal} of $(k_m)_{m=1}^M$-s is sufficient  \citep{szabo18characteristic2}. 
HSIC has been deployed successfully in a wide range of domains including independence testing \citep{gretton08kernel,pfister18kernel,albert22adaptive},  feature selection \citep{camps10remote,song12feature,wang22rank} with applications in biomarker detection \citep{gonzalez19block} and wind power prediction \citep{bouche23wind}, clustering \citep{song07dependence,gonzalez19block}, and causal discovery \citep{mooij16distinguishing,pfister18kernel,chakraborty19distance,scholkopf21causal}.

Various estimators for HSIC and other dependence measures exist in the literature, out of which we summarize the most closely related ones to our work in Table~\ref{tab:comparison}. 
The classical V-statistic based HSIC estimator (V-HSIC; \citealt{gretton05measuring,quadrianto09kernelized,pfister18kernel}) is powerful but its runtime increases quadratically with the number of samples, which limits its applicability in large-scale settings. 
To tackle this severe computational bottleneck, approximations of HSIC (N-HSIC, RFF-HSIC) have been proposed \citep{zhang18large}, relying on the Nyström \citep{williams01using} and the random Fourier feature (RFF; \citealt{rahimi07random}) method, respectively. 
However, these  estimators (i) are limited to two components, (ii) their extension to more than two components is not straightforward, and (iii) they lack theoretical guarantees. 
The RFF-based approach is further restricted to finite-dimensional Euclidean domains and to translation-invariant kernels. 
The normalized finite set independence criterion (NFSIC; \citealt{jitkrittum17adaptive2}) replaces the RKHS norm of HSIC with an $L_2$ one which allows the construction of linear-time estimators. 
However, NFSIC is also limited to two components, requires $\R^d$-valued input, and analytic kernels \citep{chwialkowski15fast}. 
Novel complementary approaches are the kernel partial correlation coefficient (KPCC; \citealt{huang22kernel}), and tests basing on incomplete U-statistics \citep{schrab22incompleteu}. One drawback of KPCC is its cubic runtime complexity w.r.t.\ the sample size when applied to kernel-enriched domains. \citet{schrab22incompleteu}'s approach can run in linear time, but it is limited to $M=2$ components. We note that all approaches require choosing an appropriate kernel: Here, one can optimize over various parametric families of kernels for increasing a proxy of test power in case of MMD \citep{jitkrittum16interpretable,liu20learning}, and in case of HSIC \citep{jitkrittum17adaptive2}. One can also design (almost) minimax-optimal MMD-based two-sample tests  using spectral regularization \citep{hagrass22spectral}.

\begin{table*}
  \centering
  \caption{Comparison of kernel independence measures: $n$ -- number of samples, $M$ -- number of components, $n'$ -- number of Nyström samples, $s$ -- number of random Fourier features, $d$ -- data dimensionality.}
  \label{tab:comparison}
  \begin{tabular}{lllll}
    \toprule
    Independence Measure                            & Runtime Complexity                 & $M$     & Domain & Admissible Kernels    \\
    \midrule
    V-HSIC \citep{pfister18kernel}       & $\O\left(Mn^2\right) $             & $M\ge2$ & any    & universal             \\
    NFSIC \citep{jitkrittum17adaptive2} & $\O\left(n\right)$                                &  $M=2$       & $\R^d$       & analytic, characteristic              \\
    
    N-HSIC \citep{zhang18large}                & $\O\left({n'}^3 + n{n'}^2\right) $ & $M=2$   & any    & characteristic        \\
    RFF-HSIC \citep{zhang18large}                    & $\O\left(s^2n\right) $                                & $M=2$   & $\R^d$ & translation-invariant, characteristic \\
    KPCC \citep{huang22kernel}   & $\O\left(n^3\right)$               & $M=2$   & any    & characteristic        \\
    \textbf{Nyström $M$-HSIC} (N-MHSIC)                     & $\O\left(M{n'}^3+Mn'n\right) $      & $M\ge2$ & any    & universal             \\
    \bottomrule
  \end{tabular}
\end{table*}

The restriction of existing HSIC approximations to two components is a severe limitation in recent applications like causal discovery which require independence tests capable of handling more than two components. Furthermore, the emergence of large-scale data sets necessitates algorithms that scale well in the sample size. To alleviate these bottlenecks, we make the following \tb{contributions}. 
\begin{enumerate}[labelindent=0em,leftmargin=1.2em,topsep=0cm,partopsep=0cm,parsep=0cm,itemsep=2mm]
    \item We propose Nyström $M$-HSIC, an efficient HSIC estimator, which can handle more than two components and has runtime $\O\left(M{n'}^3+Mn'n\right)$, where $n$ denotes the number of samples, $n' \ll n$ stands for the number of Nyström points, and $M$ is the number of random variables whose independence is measured.
    \item We provide theoretical guarantees for Nyström $M$-HSIC: we prove that our estimator converges with rate $\O\left(n^{-1/2}\right)$ for $n' = \tilde \Theta\left( \sqrt{n} \right)$, which matches the convergence of the quadratic-time estimator.
    \item We perform an extensive suite of experiments to demonstrate the efficiency of Nyström $M$-HSIC. These applications include dependency testing of media annotations and causal discovery. In the former, we achieve similar runtime and power as existing HSIC approximations. The latter requires testing joint independence of more than two components, which is beyond the capabilities of existing HSIC accelerations. Here, the proposed algorithm achieves the same performance as the quadratic-time HSIC estimator V-HSIC with a significantly reduced runtime.
\end{enumerate}

The paper is structured as follows. Our notations are introduced in Section~\ref{sec:notations}. The existing Nyström-based HSIC approximation for two components is reviewed in Section~\ref{sec:estimation-hsic}. Our proposed method, which is capable of handling $M\ge 2$ components, is presented in Section~\ref{sec:prop-hsic-appr} together with its theoretical guarantees. In Section~\ref{sec:experiments} we demonstrate the applicability of Nyström $M$-HSIC. All the proofs of our results are available in the supplementary material.

\section{Notations} \label{sec:notations}
This section is dedicated to definitions 
and to the introduction of our target quantity Hilbert-Schmidt independence criterion (HSIC). 
In particular, we introduce the \tb{notations} $[M]$, $\langle\b v, \b w\rangle$, $\norm{\b v}{2}$, $\circ_{m\in [M]} \b A_m$, $\trace(\b A)$, $\b A^{-1}$, $\b A^{-}$, $\b A\tran$, $\norm{\b A}{\text{F}}$, $\b 1_d$, $\b I_d$, $\Span$, $\mathcal M_1^+(\X{})$, $\H_k$, 
$\mu_k$, $\mathrm{MMD}_k$, $\otimes_{m=1}^M k_m$, $\otimes_{m=1}^M\P_ m$, $\mathrm{HSIC}_{\otimes_{m=1}^M k_m}$, $C_X$, $A^{-1}$, $\opnorm{A}$, $\trace(A)$, $\mathcal N_{X}(\lambda)$, $\O_{\text{P}}\left(r_n\right)$.

For a positive integer $M$, $[M] := \{1,\dots,M\}$. The Euclidean inner product of vectors $\b v, \b w \in \R^d$ is denoted by $\langle \b v, \b w \rangle $; the Euclidean norm is $\norm{\b v}{2}:= \sqrt{\langle\b v, \b v\rangle}$. The Hadamard product of matrices $\b A_m \in \R^{d_1 \times d_2}$ of equal size ($m\in [M]$) is $\circ_{m\in [M]} \b A_m := \left[\prod_{m\in [M]} (\b A_m)_{i,j}\right]_{i\in [d_1], j\in [d_2]}$. Matrix multiplication takes precendence over the Hadamard one.  For a matrix $\b A \in \R^{d\times d}$, $\trace(\b{A}):=\sum_{i \in [d]}A_{i,i}$ denotes its trace, $\b A^{-1}$ is its inverse (assuming that $\b A$ is non-singular), and $\b A^-$ is its Moore–Penrose inverse. The transpose of a matrix $\b A\in \R^{d_1 \times d_2}$ is denoted by $\b A\tran$. The Frobenius norm of a matrix $\b A \in \R^{d_1\times d_2}$ is  $\norm{\b A}{\text{F}}:=\sqrt{\sum_{i\in [d_1],j\in [d_2]} (A_{i,j})^2}$. The $d$-dimensional vector of ones is $\b 1_d$. The $d\times d$-sized identity matrix is denoted by $\b I_d$. For a set $S$ in a vector space, $\Span(S)$ denotes the linear hull of $S$. Let $(\X{},\tau_{\X{}})$ be a topological space, and $\mathcal B(\tau_{\X{}})$ the Borel sigma-algebra induced by the topology $\tau_{\X{}}$. All probability measures in the manuscript are meant with respect to the measurable space $\left(\X{},\mathcal B(\tau_{\X{}})\right)$, and they are denoted by $\mathcal M_1^+(\X{})$. The RKHS $\H_k$ on $\X{}$ associated with a kernel $k : \X{} \times \X{} \to \R$ is the Hilbert space of functions $h : \X{} \to \R$ such that $k(\cdot, x) \in \H_k$ and $\ip{h,k(\cdot,x)}{\H_k} = h(x)$ for all $x \in \X{}$ and $h\in \H_k$.\footnote{$k(\cdot,x)$ stands for $x' \in \X{} \mapsto k(x',x) \in \R$ with $x \in \X{}$ fixed.} Kernels are assumed to be bounded (in other words, there exists $B\in \R$ such that $\sup_{x,x'\in \X{}}k(x,x')\le B$) and measurable, and $\H_k$ is assumed to be separable throughout the paper.\footnote{The separability of $\H_k$ can be guaranteed on a separable topological space $\X{}$ by  taking a continuous kernel $k$ \citep[Lemma 4.33]{steinwart08support}.} The function defined by $\phi_k(x) := k(\cdot, x)$ is the canonical feature map; with this feature map $k(x,x') = \ip{k(\cdot,x),k(\cdot,x')}{\H_k} = \ip{\phi_k(x),\phi_k(x')}{\H_k}$ for all $x, x' \in  \X{}$.
A kernel $k : \R^d \times \R^d \to \R$ is called translation-invariant if there exists a function $\kappa : \R^d\to\R$ such that $k(\b x,\b x') = \kappa(\b x-\b x')$ for all $\b x,\b x' \in \R^d$. 
The mean embedding  $\mu_k$ of a probability measure $\P{} \in \mathcal M_1^+(\X{})$ is
$\kme {} := \int_{\X {}} \phi_k(x) \d\P{}(x)$, where the integral is meant in Bochner's sense. The resulting (semi-)\-metric  is called maximum mean discrepancy (MMD):
\begin{align*}
\mathrm{MMD}_k(\P, \Q) &:= \norm{\mu_k(\P{})-\mu_k(\mathbb Q)}{\H_k},
\end{align*}
for $\P{},\Q \in \mathcal M_1^+(\X{})$. The injectivity of the mean embedding $\mu_k$ is equivalent to $\mathrm{MMD}_k$ being a metric; in this case the kernel $k$ is called characteristic. Let $X = (X_m)_{m=1}^M$ denote a random variable with distribution $\P{} \in \mathcal M_1^+(\X{})$ on the product space  $\X{} = \times_{m=1}^M \X_m$, where $\X_{m}$ is enriched with kernel $k_m : \X_{m} \times \X_{m} \to \R$. The distribution of the $m$-th marginal $X_m$ of $X$ is denoted by $\P_ m \in \mathcal M_1^+(\X_m)$; the product of these $M$ marginals is $\otimes_{m=1}^M\P_ m \in \mathcal M_1^+(\X{})$.
 The tensor product of the kernels~$(k_m)_{m=1}^M$
\begin{align*}
\otimes_{m=1}^M k_m\left((x_m)_{m=1}^M,(x'_m)_{m=1}^M\right) := \prod_{m\in[M]} k_m(x_m,x_m'), 
\end{align*}
with $x_m,x_m' \in \X_m$ ($m\in [M]$), is also a kernel; we will use the shorthand $k=\otimes_{m=1}^M k_m$. The associated RKHS has a simple structure $\H_{k} = \tphs$ \citep{berlinet04reproducing} with the r.h.s.\ denoting the tensor product of the RKHSs $(\H_{k_m})_{m=1}^M$. Indeed, for $h_m \in \H_{k_m}$, the multi-linear operator $\otimes_{m=1}^M h_m \in \tphs$ acts as $\otimes_{m=1}^Mh_m(v_1,\dots,v_M)  = \prod_{m\in [M]}\ip{h_m,v_m}{\H_{k_m}}$, where $h_m,v_m \in \H_{k_m}$. The space $\tphs$ is the closure of 
the linear combination of such $\otimes_{m=1}^M h_m$-s: 
\begin{align*}
\tphs = \widebar{\Span}\left(\otimes_{m=1}^M h_m\, :\, h_m\in \H_{k_m}, m\in [M]\right),
\end{align*}
where the closure is meant w.r.t.\ to the (linear extension of the) inner product defined as 
\begin{align}
    \MoveEqLeft \left\langle \otimes_{m=1}^M a_m, \otimes_{m=1}^M b_m \right\rangle_{\tphs} \hspace{-0.1cm} := \nonumber \\
    & = \prod_{m\in [M]} \hspace{-0.1cm} \left\langle a_m,b_m\right\rangle_{\H_{k_m}},\quad a_m,b_m\in \H_{k_m}. \label{eq:tensor:inner-product} 
\end{align}
Specifically,  \eqref{eq:tensor:inner-product} implies that 
\begin{align}
\left\| \otimes_{m=1}^M a_m\right\|_{\tphs} & = \prod_{m\in [M]} \left\| a_m\right\|_{\H_{k_m}}. \label{eq:tensor:norm}
\end{align}
One can define an independence measure, the so-called Hilbert-Schmidt independence criterion based on $k$ as 
\begin{align}
  \label{eq:def-hsic}
  \mathrm{HSIC}_k(\P{}) := \mathrm{MMD}_{k}\left(\P{},\prodmarginals\right) = \norm{C_X}{\H_k},
\end{align}
where $C_X := \mu_{k}(\P{}) - \mu_{k}\left(\prodmarginals\right)$ is the centered cross-covariance operator.

Let $A : \H_k \to \H_k$ be a bounded linear operator. Its inverse  (provided that it exists) $A^{-1}: \H_k \to \H_k$ is also bounded linear. The operator norm of $A$ is defined as $\opnorm{A} := \sup_{\norm{h}{\H_k}=1} \norm{Ah}{\H_k}$. As $\H_k$ is separable, it has a countable orthonormal basis $(e_j)_{j\in J}$. $A$ is called trace-class if $\sum_{j\in J}\ip{(A^*A)^{\frac{1}{2}} e_j, e_j}{\H_k}<\infty$ where $(\cdot)^*$ denotes the adjoint, and in this case $\trace(A):= \sum_{j\in J}\ip{A e_j, e_j}{\H_k} < \infty$ is called the trace of $A$.
For $\P{} \in  \mathcal M_1^+(\X{})$, kernel $k:\X{}\times \X{} \rightarrow \R$ and $\lambda >0$, the  uncentered covariance operator  is $\mu_{k\otimes k}(\P{}):=\int_{\X} k(\cdot, x) \otimes k(\cdot,x) \d \P(x)$ and its regularized variant is $\mu_{k\otimes k,\lambda}(\P{}) := \mu_{k\otimes k}(\P{}) + \lambda I$, respectively, where $I$ denotes the identity operator. Let $\mathcal N_x(\lambda)= \ip{\phi_k(x),\mu^{-1}_{k\otimes k,\lambda}(\P{})\phi_k(x)}{\H_{k}}$. The effective dimension of $X \sim \P{}$ is defined as $\mathcal N_{X}(\lambda) := \E_{x\sim \P{}}\left[\mathcal N_{x}(\lambda)\right] = \trace\left(\mu_{k\otimes k}(\P{})\mu_{k\otimes k,\lambda}^{-1}(\P{})\right)$. For a sequence of $r_n>0$-s and a sequence of real-valued random variables $X_n$, $X_n = \O_{\text{P}}(r_n)$ denotes that $\frac{X_n}{r_n}$ is bounded in probability.

\section{Existing HSIC Estimators}
\label{sec:estimation-hsic}

We recall the existing HSIC estimator V-HSIC in Section~\ref{sec:class-hsic-estim}, and its Nyström approximation for two components in Section~\ref{sec:nystr-estim-m=2}. We present our proposed Nyström approximation for more than two components in Section~\ref{sec:prop-hsic-appr}.

\subsection{Classical HSIC Estimator (V-HSIC)} \label{sec:class-hsic-estim}

Given an i.i.d.\ sample of $M$-tuples of size $n$
\begin{align}
\Psamp_{n} := \left\{\left(x_1^1, \dots, x_M^1\right), \dots, \left(x_1^n, \dots, x_M^n\right)\right\} \subset \X{} \label{eq:sample-of-m-tuplets}
\end{align}
drawn from $\P{}$, the corresponding  empirical estimate of the squared HSIC, obtained by replacing the population means with the sample means, gives rise to the V-statistic based estimator
\begin{eqnarray}
 \lefteqn{0\le\HSIC_{k}^2\left(\hat \P_n\right) := \frac {1}{n^2}\ \bm 1_n\tran\left(\hadamard_{m\in[M]}\mathbf{K}_{k_m}\right) \bm 1_n} \label{eq:emp-hsic}\\
 && \hspace{-0.6cm}+ \frac{1 }{n^{2M}} \prod_{m\in[M]}\bm 1_n\tran \mathbf{K}_{k_m} \bm 1_n - \frac {2}{n^{M+1}}\bm 1_n\tran \left( \hadamard_{m\in[M]}\mathbf{K}_{k_m} \bm 1_n\right) \nonumber
 \end{eqnarray}
with Gram matrices
\begin{align}
\mathbf{K}_{k_m} = \left[k_m\left(x_m^i,x_m^j\right)\right]_{i,j \in [n]} \in \R^{n\times n},\label{eq:gram-matrix}
\end{align}
which can be computed in $O(n^2M)$.\footnote{$\HSIC_{k}^2(\hat \P_n)$ denotes the application of $\HSIC_k^2$ to the empirical measure $\hat \P_n$. $\HSIC^2_{k,\text N_0}(\hat \P_n)$ and $\HSIC_{k,\text N}^2(\Psamp_{n})$ indicate dependence on $\hat \P_n$. Similarly, $\mu_{\ell}(\hat \Q_{n})$ stands for application, $\mu_\ell(\tilde \Q_{n'})$, $\mu_{k_m}(\tilde \P_{m,n'})$ and $\mu_{k}(\tilde \P_{n'})$ indicate dependence on the argument. \label{footnote:app-vs-dependence}}
This prohibitive runtime inspired the development of HSIC approximations \citep{zhang18large} using the Nyström method and random Fourier features. We review the Nyström-based construction in Section~\ref{sec:nystr-estim-m=2} and explain why the technique is restricted to $M=2$ components, before presenting our alternative approximation scheme of HSIC in Section~\ref{sec:prop-hsic-appr} which is capable of handling $M\geq2$ components.

\subsection{Nyström Method} \label{sec:nystr-estim-m=2}
In this section, we recall the existing Nyström approximation, which can handle $M=2$ components.

The expression \eqref{eq:emp-hsic} can be rewritten \citep{gretton05measuring} for $M=2$ components as 
\begin{align}
  \label{eq:hsic-two-components}
  \HSIC^2_{k}\left(\Psamp_{n}\right) = \frac{1 }{n^2}\trace\left( \mathbf{H  K}_{k_1} \mathbf{ HK}_{k_2} \right),
\end{align}
with the centering matrix $ \mathbf{H} =  \mathbf{I}_n - \frac 1 n \bm 1_n \bm 1_n \tran \in \R^{n\times n}$, Gram matrices $\mathbf{K}_{k_1},\,\mathbf{K}_{k_2}$ defined in \eqref{eq:gram-matrix}, and sample $\Psamp_n := \left\{(x_1^1,x_2^1),\dots,(x_1^n,x_2^n)\right\}$ as in \eqref{eq:sample-of-m-tuplets} with $M=2$. The naive computation of \eqref{eq:hsic-two-components} costs $\O\left(n^3\right)$. However, noticing that $\trace(\b A\tran \b B) = \sum_{i,j\in[n]}  A_{i,j } B_{i,j}$, the computational complexity reduces to $\O\left(n^2\right)$. The quadratic complexity can be reduced by the Nyström approximation\textsuperscript{\ref{footnote:app-vs-dependence}} \citep{zhang18large} 
\begin{align}
\label{eq:hsic-ny-plugin}
\begin{split}
  \HSIC^2_{k,\text N_0}\left(\hat \P_n\right) &= \frac{1}{n^2}\trace\left(\mathbf{HK}^{\nys}_{k_1}\mathbf{HK}^{\nys}_{k_2}\right) \\
  &\stackrel{(*)}{=} \frac {1}{n^2}\norm{\left(\mathbf{H}\phi_{k_1}^{\nys}\right)\tran \mathbf{H}\phi_{k_2}^{\nys}}{\text F}^2,
\end{split}
\end{align}
which we detail in the following. The Nyström approximation relies on a subsample of size $n' \leq n$ of $\Psamp_n$, which we denote by $\tilde\P_{n'} :=  \left\{\big(\tilde x_1^1, \tilde x_2^1\big), \dots,  \big(\tilde x_1^{n'}, \tilde x_2^{n'}\big)\right\}$; the tilde indicates a relabeling. The subsample allows to define three matrices
\begin{align}
  \label{eq:matrix-definitions}
  \begin{split}
      \mathbf{K}_{k_m,n'n'} &= \left[k_m\left(\tilde x_m^i,\tilde x_m^j\right)\right]_{i,j\in[n']} \in \R^{n'\times n'}, \\
  \mathbf{K}_{k_m,nn} &= \mathbf{K}_{k_m} \in \R^{n\times n}, \\
  \mathbf{K}_{k_m,n'n} &= \left[k_m(\tilde x_m^i, x_m^j)\right]_{i\in[n'],j\in[n]} \in \R^{n'\times n}, 
  \end{split}
\end{align}
where $m\in [2]$ and $\b K_{k_m}$ is defined in \eqref{eq:gram-matrix}, and let $\mathbf{K}_{k_m,nn'} = \mathbf{K}_{k_m,n'n}\tran \in \R^{n\times n'}$. The matrices  $\b K_{k_m}^{\nys}$ $(m \in [2])$ as used in \eqref{eq:hsic-ny-plugin} are 
\begin{align*}
  \mathbf{K}_{k_m}^{\nys} &:= \mathbf{K}_{k_m,nn'}\mathbf{K}_{k_m,n'n'}^{-1}\mathbf{K}_{k_m,n'n}
  \\ &= \underbrace{\mathbf{K}_{k_m,nn'}\mathbf{K}_{k_m,n'n'}^{-\frac 1 2 }}_{=:\phi_{k_m}^{\nys} \in \R^{n\times n'}} \big(\underbrace{\mathbf{K}_{k_m,nn'}\mathbf{K}_{k_m,n'n'}^{-\frac 1 2 }}_{\phi_{k_m}^{\nys}}\big)\tran  \hspace{-0.1cm}\in \R^{n\times n}, 
\end{align*}
provided that the inverse $\mathbf{K}_{k_m,n'n'}^{-1}$ exists.
In \eqref{eq:hsic-ny-plugin} the r.h.s.\ of $(*)$ has a computational complexity of $\O({n'}^3 + nn'^2)$,\footnote{This follows from the complexity of $O({n'}^3)$ of inverting an $n' \times n'$ matrix and the complexity of multiplying both feature representations \citep{zhang18large}.} which is smaller than $\O\left(n^2\right)$ of \eqref{eq:hsic-two-components},  provided that $n' < \sqrt n$; this speeds up the computation. $(*)$ relies on the  cyclic invariance property of the trace, and the idempotence of $\b H$ (in other words, $\b H \b H = \b H$), limiting the above derivation to $M=2$ components; the approach does not extend naturally to the case of $M>2$.

\section{Proposed HSIC Estimator}
\label{sec:prop-hsic-appr}
We now elaborate the proposed Nyström HSIC approximation for $M\geq 2$ components.

Recall that the centered cross-covariance operator takes the form
\begin{align}
  C_X &= \mu_{k}(\P{}) - \mu_{k}\left(\prodmarginals\right)  \nonumber \\ 
  &
  = \mu_k(\P{}) - \otimes_{m=1}^M\mu_{k_m}\left(\P_m\right). \label{eq:centered-cross-cov}
\end{align}
There are $M+1$ expectations in this expression; we  estimate these mean embeddings separately. This conceptually simple construction, is to the best of our knowledge, the first that handles $M\ge 2$ components, and it allows to leverage recent bounds on mean estimators (Lemma~\ref{lemma:nystrom-mean-embedding}).
We first detail the general Nyström method for approximating expectations  $\int_{\Y{}}\phi_\ell(y)\d\Q{} (y) $ associated to a kernel $\ell : \Y \times \Y \to \R$ and probability distribution $\mathbb Q \in \mathcal M_1^+(\Y{})$. One can then choose
\begin{align}
\begin{split}
    (\Y, \ell, \Q) &= (\X, k, \P{}), \text{ and }\\ (\Y, \ell, \Q) &= (\X_m, k_m, \P_{m}), \quad m\in[M],
\end{split}
\label{eq:triplets}
\end{align}
to achieve our goal.

Let $\tilde \Q_{n'} = \left\{\tilde y^1,\dots,\tilde y^{n'}\right\}$ be a subsample (with replacement) of $\hat \Q_n = \left\{y^1,\dots,y^n\right\} \stackrel{\text{i.i.d.}}{\sim} \Q$ referred to as Nyström points; the tilde again indicates relabeling. 
The usual estimator of the mean embedding replaces the population mean with its empirical counterpart over $n$ samples\textsuperscript{\ref{footnote:app-vs-dependence}}
\begin{align*}
  \mu_{\ell}(\Q) = \int_{\Y{}}\phi_\ell(y)\d\Q{} (y) 
   \approx \frac 1 n \sum_{i\in[n]} \phi_\ell(y^i) =  \mu_{\ell}(\hat \Q_{n}).
\end{align*}
Instead, the Nyström approximation uses a weighted sum with weights  $\alpha_i \in \R$ ($i \in [n']$): given $n'$ Nyström points, the estimator takes the form\textsuperscript{\ref{footnote:app-vs-dependence}}
\begin{align*}
  \mu_{\ell}(\Q{})   &\approx \sum_{i\in[n' ]}\alpha_i \phi_\ell(\tilde y^i) = \mu_{\ell}\left(\tilde \Q_{n'}\right) \in  \mathcal H_\ell^{\nys},
\end{align*}
where $\H_\ell^{\nys} := \Span\left(\phi_\ell\big(\tilde y^i\big)\,:\,i\in [n']\right) \subset \mathcal H_\ell$.
The coefficients $\bm \alpha_{\ell} = (\alpha_{\ell}^1,\dots,\alpha_{\ell}^{n'})\tran \in \R^{n'}$  are obtained by the minimum norm solution of 
\begin{align}
 \min_{\bm\alpha_\ell \in\R^{n'}}\norm{\mu_{\ell}\left({\hat \Q_n}\right) -\sum_{i\in[n' ]}\alpha_i \phi_\ell\left(\tilde y^i\right)}{\H_\ell}^2. \label{eq:optim-prob}
\end{align}
The following lemma  describes the solution of \eqref{eq:optim-prob}.

\begin{lemma}[Nyström mean embedding, \citet{chatalic22nystrom}] 
  \label{lemma:nystrom-mean-embedding}
For a kernel $\ell$ with corresponding feature map $\phi_\ell$, an i.i.d.\  sample $\hat\Q_n$ of distribution $\Q{}$, and a subsample $\tilde \Q_{n'}$ of $\hat \Q_n$, the Nyström estimate of $\mu_{\ell}(\Q{})$ is given by
  \begin{align}
    \mu_{\ell}\left(\tilde \Q_{n'}\right) &= \sum_{i \in[n']}\alpha_\ell^i\phi_\ell\left(\tilde y^i\right), \nonumber\\
    \bm \alpha_\ell &= \frac 1 n \left(\mathbf{ K}_{\ell, n'n'}\right)^{-}\mathbf{K}_{\ell, n'n}\bm 1_n  \label{eq:alpha-k},
  \end{align}
  with Gram matrix  $\mathbf{K}_{\ell,n'n'} = \left[\ell(\tilde x^i,\tilde x^j)\right]_{i,j\in[n']} \in \R^{n'\times n'}$, and  $\mathbf{K}_{\ell,n'n} = \left[ \ell(\tilde x^i, x^j) \right]_{i\in[n'],j\in[n]} \in \R^{n'\times n}$.
\end{lemma}

Let
\begin{align}
\tilde\P_{n'} = \left\{\Big(\tilde x_1^1,\dots,\tilde x_M^1\Big),\dots,\Big(\tilde x_1^{n'},\dots,\tilde x_M^{n'}\Big)\right\} \label{eq:Nystrom-samples}
\end{align}
be a subsample (with replacement) of
$
\Psamp_n = \left\{\left(x_1^1,\dots,x_M^1\right),\dots,\left(x_1^n,\dots,x_M^n\right)\right\}
$ defined in \eqref{eq:sample-of-m-tuplets}, and
\begin{align} \label{eq:tildeP:m,n'}
\tilde\P_{m,n'} = \left\{\tilde x_m^1,\dots,\tilde x_m^{n'}\right \}
\end{align}
be the corresponding subsample of the $m$-th marginal ($m\in [M]$).
Using our choice \eqref{eq:triplets} with Lemma~\ref{lemma:nystrom-mean-embedding}, the estimators for the embeddings of marginal distributions take the form\textsuperscript{\ref{footnote:app-vs-dependence}}
\begin{align}
    \mu_{k_m}\left(\tilde \P_{m,n'}\right) &= \sum_{i \in [n']}\alpha_{k_m}^i \phi_{k_m}\left(\tilde x_m^i\right), \nonumber\\
  \bm \alpha_{k_m} &= \frac 1 n \left( \mathbf{K}_{k_m,n'n'}\right)^{-} \mathbf{K}_{k_m,n'n} \bm 1_n,  \label{eq:alpha-marginal}
\end{align}
and the estimator of the mean embedding of the joint distribution is\textsuperscript{\ref{footnote:app-vs-dependence}}
\begin{align}
  \mu_{k}\left(\tilde \P_{n'}\right) &= \sum_{i \in [n']} \alpha_k^i \otimes_{m=1}^M \phi_{k_m}\left(\tilde x_m^i\right), \nonumber
    \\\bm \alpha_{k} &=  \frac 1 n \left(  \mathbf{K}_{k,n'n'} \right)^{-}\left( \mathbf{K}_{k,n'n}\right) \bm 1_n \nonumber  
  \\
  & \stackrel{(*)}{=}  \frac 1 n \overbrace{\Big(\underbrace{\hadamard_{m\in[M]} \mathbf{K}_{k_m,n'n'}}_{(a)}\Big)^{-}}^{(c)} \times \nonumber \\ &\quad\hspace{.55cm} \Big(\underbrace{\hadamard_{m\in[M]}  \mathbf{K}_{k_m,n'n}}_{(b)}\Big) \bm 1_n \label{eq:alpha-prod},
\end{align}
where $(*)$ holds as for the Gram matrix $\b K_{k,n'n'}$ associated with the product kernel $k = \otimes_{m\in [M]}k_m$ one has 
\begin{align*}
  \mathbf{K}_{k,n'n'} &= \left[ k\left((\tilde x_1^i,\dots,\tilde x_M^i),(\tilde x_1^j,\dots,\tilde x_M^j)\right)\right]_{i,j\in[n']} \\ &= \left[\prod_{m\in[M]}k_m(\tilde x_m^i,\tilde x_m^j)\right]_{i,j\in[n']} \hspace{-0.6cm}= \hadamard_{m\in[M]}\mathbf{K}_{k_m,n'n'}, 
\end{align*}
and similarly $\mathbf{K}_{k,n'n} = \hadamard_{m\in[M]}\mathbf{K}_{k_m,n'n}$, with $\mathbf{K}_{k_m,n'n'}$ and $\mathbf{K}_{k_m,n'n}$ defined in \eqref{eq:matrix-definitions}.

Combining the $M+1$ Nyström estimators in \eqref{eq:alpha-marginal} and in \eqref{eq:alpha-prod} gives rise to the overall Nyström HSIC estimator, which is elaborated in the following lemma.

\begin{lemma}[Computation of Nyström $M$-HSIC]
  \label{thm:nystroem-hsic}
  The Nyström estimator for HSIC  can be expressed as\textsuperscript{\ref{footnote:app-vs-dependence}}
  \begin{eqnarray}
    \label{eq:nyström-hsic}
    \lefteqn{\HSIC_{k,\text N}^2\left(\Psamp_{n}\right) = \bm \alpha_k\tran\left( \hadamard_{m\in[M]} \mathbf{K}_{k_m,n'n'}\right) \bm \alpha_k } \\
                                               && \hspace{-0.8cm} + \hspace{-0.2cm}\prod_{m\in[M]} \bm \alpha_{k_m}\tran  \mathbf{K}_{k_m,n'n'} \bm\alpha_{k_m} 
                                                \hspace{-0.1cm} - 2 \bm\alpha_k\tran \left ( \hadamard_{m\in[M]}  \mathbf{K}_{k_m,n'n'}\bm\alpha_{k_m} \right), \nonumber
  \end{eqnarray}
  with $\bm \alpha_{k_m}$ and $\bm \alpha_{k}$  defined in \eqref{eq:alpha-marginal} and  \eqref{eq:alpha-prod}, respectively, $\mathbf{K}_{k_m,n'n'}$ is defined in \eqref{eq:matrix-definitions}, and $N$ in the subscript of the estimator refers to Nyström. Note that \eqref{eq:nyström-hsic} depends on $\hat \P_{n}$ as one must solve \eqref{eq:optim-prob}.
\end{lemma}

\begin{remark}~
  \label{remark:nystroem-hsic}
  \begin{itemize}[labelindent=0em,leftmargin=0.85em,topsep=0cm,partopsep=0cm,parsep=0cm,itemsep=2mm]
  \item \textbf{Uniform weights, no subsampling.} The estimator \eqref{eq:nyström-hsic} gives back \eqref{eq:emp-hsic} when $\bm \alpha_k :=  \bm \alpha_{k_m} := \frac 1 n \bm 1_n$ for all $m\in[M]$, and when there is no subsampling applied.
  
  \item \textbf{Runtime complexity.}   In order to determine the computational complexity of \eqref{eq:nyström-hsic} one has to find that of \eqref{eq:alpha-prod}; that of \eqref{eq:alpha-marginal} follows by choosing $M=1$ in \eqref{eq:alpha-prod}.
    $(a)$ and $(b)$ in \eqref{eq:alpha-prod} are Hadamard products; hence their computational complexity is $\O\left(M{n'}^2\right)$ and $\O\left(Mnn'\right)$. $(c)$ in \eqref{eq:alpha-prod} is the Moore-Penrose inverse of an $n' \times n'$ matrix; thus its complexity is $\O\left({n'}^3\right)$. Hence, the computation of $\bm \alpha_k$ costs $\O\big(M{n'}^2 + {n'}^3+Mn'n \big)$, and that of $(\bm \alpha_{k_m})_{m=1}^M$ is  $\O\big({n'}^2 + {n'}^3+n'n \big)$ for each $m\in[M]$.
    In \eqref{eq:nyström-hsic} each term can be computed in $\O\left(M {n'}^2\right)$. Overall the Nyström $M$-HSIC estimator has complexity  $\O\big(M{n'}^2 + M{n'}^3+Mn'n \big) = \O\big(M{n'}^3+Mn'n \big)$.
    
  \item \textbf{Difference compared to the estimator by \citet{zhang18large}.} For $M=2$, \eqref{eq:nyström-hsic} reduces to
    \begin{eqnarray}
      \label{eq:ours-two-components}
      \lefteqn{
      \HSIC_{k,\text N}^2\left( \Psamp_{n}\right) = \bm\alpha_k\tran\left(\circ_{i\in[2]}\mathbf{K}_{k_i,n'n'}\right) \bm\alpha_k  } \\
 &&\hspace{-0.5cm}+ \bm\prod_{i\in[2]}\alpha_{k_i}\tran \mathbf{K}_{k_i,n'n'} \bm\alpha_{k_i}  -2 \bm\alpha_k\tran\left(\circ_{i\in[2]} \mathbf{K}_{k_i,n'n'}\bm\alpha_{k_i}\right). \nonumber
    \end{eqnarray}
    Using the equivalence of \eqref{eq:emp-hsic} and  \eqref{eq:hsic-two-components} in case $M=2$ gives
    \begin{eqnarray*}
      \lefteqn {\trace\left(\mathbf{HK}_{k_1}\mathbf{HK
      }_{k_2}\right) = \frac{1}{n^2}\bm1_n\tran\left(\mathbf{K}_{k_1} \circ \mathbf{K}_{k_2}\right)\bm 1_n} \\
      &&\hspace{-0.5cm}+ \frac{1 }{n^4}\prod_{i\in[2]}\bm 1_n\tran \mathbf{K}_{k_i} \bm 1_n - \frac{2}{n^3}\bm 1_n \tran\left(\mathbf{K}_{k_1}\bm1_n\circ \mathbf{K}_{k_2}\bm 1_n\right),
    \end{eqnarray*}
    hence \eqref{eq:hsic-ny-plugin} becomes
    \begin{eqnarray}
      \label{eq:zhang-extended-form}
      \lefteqn{\HSIC_{k,\text N_0}^2\left( \Psamp_{n}\right)  = \frac{1}{n^2}\bm1_n\tran\left(\mathbf{K}_{k_1}^{\nys} \circ \mathbf{K}_{k_2}^{\nys}\right)\bm 1_n} \\ 
      &&\hspace{-0.5cm}+ \frac{1 }{n^4} \prod_{i\in [2]}\bm 1_n\tran \mathbf{K}_{k_i}^{\nys} \bm 1_n - \frac{2}{n^3}\bm 1_n \tran\left(\mathbf{K}_{k_1}^{\nys}\bm1_n\circ \mathbf{K}_{k_2}^{\nys}\bm 1_n\right).\nonumber
    \end{eqnarray}
    The estimators \eqref{eq:ours-two-components} and \eqref{eq:zhang-extended-form} are identical if $\bm \alpha_k =  \bm \alpha_{k_m} = \frac 1 n \bm 1_n$ for all $m\in[M]$ and when there is no subsampling; in the general case they do not coincide. In \eqref{eq:hsic-ny-plugin} the dominant term in the complexity is $\left(n'\right)^2 n$ (since $n' < n$), this reduces to $n'n$ in our proposed estimator \eqref{eq:nyström-hsic}.

  \end{itemize}
\end{remark}

Key to showing the  consistency of the proposed Nyström $M$-HSIC estimator \eqref{eq:nyström-hsic} (Proposition~\ref{thm:error-nystrom-hsic})
is our next lemma, which describes how the Nyström approximation error of the mean embeddings of the components ($d_{k_m}$ below) can be propagated through tensor products.

\begin{lemma}[Error propagation on tensor products]
  \label{lemma:decomposition}
  Let $X=(X_m)_{m=1}^M  \in \X = \times_{m=1}^M \X_m$, $k_m: \X_m \times \X_m \rightarrow \R$ bounded kernels ($\exists a_{k_m} \in (0,\infty)$ such that $\sup_{x_m\in \X_m}\sqrt{k_m(x_m,x_m)} \leq a_{k_m}$, $m\in [M]$), $k=\otimes_{m=1}^M k_m$, $\H_k$ the RKHS associated to $k$,  $X\sim\P{}\in \mathcal M_1^+(\X{})$, $\P_m$ the $m$-th marginal of $\P$ ($m\in [M]$), $n'\le n$, and $\tilde\P_{m,n'}$ defined according to \eqref{eq:tildeP:m,n'}.
   Then
    \begin{eqnarray*}
      \lefteqn{\norm{\otimes_{m=1}^M \mu_{k_m}\left(\P_m\right) - \otimes_{m=1}^M \mu_{k_m}\left(\tilde\P_{m,n'}\right)}{\H_k} \le}\\
      &&\le \prod_{m\in[M]}\left(a_{k_m}+d_{k_m}\right) - \prod_{m\in[M]}a_{k_m},
    \end{eqnarray*}
    where $d_{k_m} = \norm{\mu_{k_m}\left(\P_m\right) -  \mu_{k_m}\left(\tilde \P_{m,n'}\right)}{\H_{k_m}}$.
  \end{lemma}
  Our resulting Nyström $M$-HSIC performance guarantee is as follows.
  \begin{proposition}[Error bound for Nyström $M$-HSIC] \label{thm:error-nystrom-hsic}
   Let $X=(X_m)_{m=1}^M \in \mathcal X = \times_{m=1}^M \X_m$, $ X \sim\P{}\in \mathcal M_1^+(\X{})$,  $(\mathcal X_m)_{m\in [M]}$ locally compact, second-countable topological spaces, $k_m: \mathcal X_m \times \mathcal X_m \rightarrow \R$ bounded kernels, i.e., $\exists a_{k_m}\in (0,\infty)$ such that $\sup_{x_m\in \X_m}\sqrt{k_m(x_m,x_m)} \leq a_{k_m}$ for all $m\in [M]$, $k=\otimes_{m\in [M]}k_m$, $a_k = \prod_{m=1}^M a_{k_m}$,  $\phi_{k_m}(x_m)=k_m(\cdot,x_m)$ for all $x_m \in \X_m$, $\phi_k = \otimes_{m=1}^M \phi_{k_m}$, $C_k = \E_{}\left[\phi_{k}(X) \otimes \phi_{k}(X) \right]$, $C_{k_m} = \E_{}\left[\phi_{k_m}(X_m) \otimes \phi_{k_m}(X_m) \right]$, the number of Nyström points $n' \le n$,  $\hat\P_{n}$ defined according to \eqref{eq:sample-of-m-tuplets}.
  Then, for any $\delta \in \left(0, \frac{1}{M+1}\right)$  
  \begin{eqnarray*}
    \lefteqn{\left|\mathrm{HSIC}_k(\P{}) - \HSIC_{k,\text N}\left(\hat \P_{n}\right) \right| \leq \underbrace{\frac{c_{k,1}}{ \sqrt{n}}}_{t_{k,1}} + \underbrace{\frac{c_{k,2}}{n'}}_{t_{k,2}} + }\\
    &&+ \underbrace{\frac{c_{k,3}\sqrt{\log(n'/\delta)}}{n'}\sqrt{\mathcal N_{X}\left(\frac{12a_k^2\log(n'/\delta)}{n'}\right)}}_{t_{k,3}} +\\
    &&+\prod_{m\in[M]}\Bigg[a_{k_m} + \underbrace{\frac{c_{k_m,1}}{ \sqrt{n}}}_{t_{k_m,1}} + \underbrace{\frac{c_{k_m,2}}{n'}}_{t_{k_m,2}} + \\
    && + \underbrace{\frac{c_{k_m,3}\sqrt{\log(n'/\delta)}}{n'}\sqrt{\mathcal N_{X_m}\left(\frac{12a_{k_m}^2\log(n'/\delta)}{n'}\right)}}_{t_{k_m,3}}\Bigg]\\
    &&-\prod_{m\in[M]}a_{k_m}
  \end{eqnarray*}
 holds  with probability at least $1-(M+1)\delta$, provided that
  \begin{align*}
    n' \geq \max_{m\in[M]}\left(67,12a_k^2\opnorm{C_k}^{-1},12a_{k_m}^2\opnorm{C_{k_m}}^{-1}\right)\log\frac{n'}{\delta},
  \end{align*}
  where $c_{k,1}= 2a_k\sqrt{2\log(6/\delta)}$, $c_{k,2}=4\sqrt 3 a_k \log(12/\delta)$, $c_{k,3}= 12\sqrt{3\log(12/\delta)}a_k$, $c_{k_m,1} =  2a_{k_m}\sqrt{2\log(6/\delta)}$, $c_{k_m,2}=4\sqrt 3 a_{k_m} \log(12/\delta)$, $c_{k_m,3} =  12\sqrt{3\log(12/\delta)}a_{k_m}$ for $m \in [M]$.
\end{proposition}

As a baseline, to interpret the result (see the second bullet point in Remark~\ref{remark:main-prop-remarks}), one could consider the V-statistic based HSIC estimator \eqref{eq:emp-hsic} for $M\ge 2$, which according to our following lemma has a convergence rate of $\O_{\text{P}}\left(\frac {1}{\sqrt n}\right)$.

\begin{lemma}[Deviation bound for V-statistic based HSIC estimator]
  \label{lemma:deviation}
  Let $\HSIC_{k}(\hat \P_n)$ be as in \eqref{eq:emp-hsic} on a metric space $\X=\times_{m=1}^M\X_m$, and $\HSIC_k\left(\P\right) > 0$. Then
  \begin{align*}
    \left|\HSIC_k\left(\P\right)-\HSIC_{k}\left(\hat \P_n\right)\right| = \O_{\text{P}}\left(\frac{1}{\sqrt n} \right).
  \end{align*}  
\end{lemma}

\begin{remark}~ \label{remark:main-prop-remarks}
  \begin{itemize}[labelindent=0em,leftmargin=0.85em,topsep=0cm,partopsep=0cm,parsep=0cm,itemsep=2mm]
  \item From the terms $t_{k,1}, t_{k,2}, t_{k_m,1}, t_{k_m,2}, m\in[M]$ it follows that for $n' < \sqrt n$ the respective second term dominates, thus increasing the error; for $n' > \sqrt n$ the respective first term dominates and the computational complexity increases. The effective dimension $\left(t_{k,3}, t_{k_m,3}\right)$ controls the trade off between the two terms and can be related \citep{chatalic22nystrom} to the decay of the eigenvalues of the respective covariance operator. A convergence rate of $n^{-1/2}$ for the sums $t_{k,1}+t_{k,2}+t_{k,3}$ and $t_{k_m,1}+t_{k_m,2}+t_{k_m,3}$ can be achieved if
    \begin{itemize}[labelindent=0em,leftmargin=1.2em,topsep=0cm,partopsep=0cm,parsep=0cm,itemsep=2mm]
    \item  $\max_{m\in[M]}\left(\mathcal N_X(\lambda),\mathcal N_{X_m}(\lambda)\right) \le c\lambda^{-\gamma}$ for some $c > 0$ and $\gamma \in (0,1]$ with $n' = n^{1/(2-\gamma)}\log(n/\delta)$, or
    \item  $\max_{m\in[M]}\left(\mathcal N_X(\lambda),\mathcal N_{X_m}(\lambda)\right) \le \log(1+c/\lambda)/\beta$ for some $c>0$, $\beta >0$, and  $n' = \sqrt n \log\left(\sqrt n \max\limits_{m\in[M]}\left(\frac 1 \delta, \frac{c}{6a_k^2}, \frac{c}{6a_{k_m}^2}\right)\right)$.
    \end{itemize}
    This rate of convergence propagates through the product.
   
  \item Lemma~\ref{lemma:deviation} establishes that the V-statistic based estimator of HSIC converges with rate $n^{-1/2}$. Recalling the last line of Table~\ref{tab:comparison}, setting $n' = o\left(n^{2/3}\right)$, the proposed estimator yields an asymptotic speedup over V-HSIC. Hence, setting $n' = \tilde \Theta\left(\sqrt n\right)$ allows to obtain the same rate of convergence while decreasing runtime. Assumption $\HSIC_k\left(\P\right) > 0$ in Lemma~\ref{lemma:deviation} protects one from attaining a convergence rate of $n^{-1}$ of $\HSIC^2_k\left(\hat \P_n\right)$.
  \end{itemize}  
\end{remark}

\section{Experiments}
\label{sec:experiments}

In this section, we demonstrate the efficiency of the proposed method (N-MHSIC) 
against the baselines NFSIC, RFF-HSIC, N-HSIC and the quadratic-time V-statistic based HSIC estimator (V-HSIC) in the context of independence testing. Hence, the null  hypothesis $H_0$ is that the joint distribution factorizes to the product of the marginals, the alternative $H_1$ is that this is not the case. The experiments study both synthetic (Section~\ref{subsec:toy-problems}) and real-world (Section~\ref{subsec:real-problems}) examples, in terms of power and runtime.\footnote{The code of our experiments is available at \url{https://github.com/FlopsKa/nystroem-mhsic}.} 

We use the Gaussian kernel 
\begin{align*}
k_m(\b x_m,\b x_m') = \exp\left(-\gamma_{k_m} \norm{\b x_m-\b x_m'}{2}^2\right)
\end{align*}
for all experiments, with $\gamma_{k_m}$ chosen according to the median heuristic. For a fair comparison of the test power, we approximate the null distribution of each test statistic by the permutation approach with $250$ samples.  We then perform a one-sided test with an acceptance region of $5\%$ ($\alpha=0.05$), which we repeat, for all power experiments, on $100$ independent draws of the data; the runtime results include these. We set each algorithm's parameters as recommended by the respective authors: For NFSIC, we set the number of test locations $J =5 $; the number of Fourier features (RFF-HSIC) and Nyström samples (N-HSIC) is set to $ \sqrt n$. The number of Nyström samples of N-MHSIC is indicated within the experiment description. The opaque area in the figures indicates the $0.95$-quantile obtained over $5$ runs. All experiments were performed on a PC with Ubuntu 20.04, 124GB RAM, and 32 cores with 2GHz each. 

\subsection{Synthetic Data}
\label{subsec:toy-problems}

We examine three toy problems in the following, illustrating runtime and statistical power.

\paragraph{Comparison of HSIC approximations under $H_0$.} \label{sec:toy-comparison}
First, for $M=2$ components, we compare our proposed method to the existing accelerated HSIC estimators (N-HSIC, RFF-HSIC) on independent data to assess convergence w.r.t.\ runtime.  Specifically, we set $X_1, X_2 \stackrel{\text{i.i.d.}}{\sim} \mathcal N(0,1)$. The theoretical value of HSIC is thus zero. 
Figure~\ref{fig:runtime_indep_data} shows the estimates for sample sizes from 100 to 1000; the number of Nyström samples for N-MHSIC is set to $n' = 2\sqrt n$. All approaches converge to zero, with N-MHSIC converging a bit slower than the exisiting HSIC approximations. However, we note that the gap is on the order of $10^{-3}$ so it is close to the theoretical value also for small sample sizes. The runtime scales as predicted by the complexity analysis, with the proposed approach running faster than both N-HSIC and RFF-HSIC starting from $n=500$ samples.

\begin{figure}
  \centering
  \includegraphics[width=\linewidth]{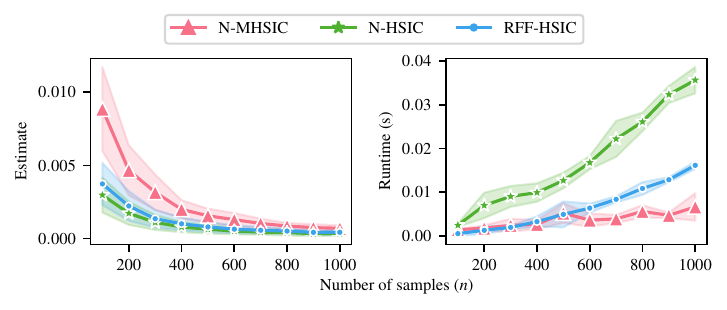}
  \caption{Estimation accuracy for $M=2$ components; the theoretical HSIC value is zero.}
  \label{fig:runtime_indep_data}
\end{figure}

\paragraph{Dependent Data ($H_1$ holds).} \label{sec:toy-strong-dependence}
To evaluate the statistical power on $M=2$ components, we set $X_1 \sim \mathcal N(0,1)$, $X_2=X_1 +\epsilon$, and $\epsilon \sim \mathcal N(0,1)$, with $n'$ set as before. Figure~\ref{fig:power} shows that N-MHSIC achieves a power of one for $n\approx 100$ and that it is slightly worse than the existing HSIC approximations for small sample sizes. V-HSIC has the highest power but also the highest runtime. Even though NFSIC has linear runtime complexity it is slower than all other statistics on small sample sizes.

\begin{figure}
  \centering
  \includegraphics[width=\linewidth]{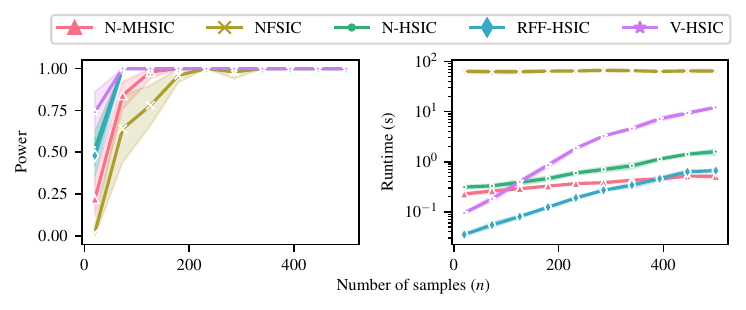}
  \caption{Power on dependent data. Runtime on log scale.}
  \label{fig:power}
\end{figure}

\paragraph{Causal Discovery.} \label{seq:toy-causal-discovery} The experiments until now considered $M=2$ components. However, N-MHSIC allows for handling $M\geq 2$ components and thus can estimate the directed acyclic graph (DAG) governing causality if one assumes an additive noise model.

Specifically, we sample from the structural equations
    $X_i = \sum_{j\in\mathrm{PA}_i}f^{i,j}\left(X_j\right)+\epsilon_i$  for $i \in [M]$,
of a randomly selected fully connected DAG with four nodes ($M=4$), of which there are $24$.
In the equation, $\mathrm{PA}_i$ denotes the parents of $i$ in the associated DAG, and the $\epsilon_i$ are normally distributed and jointly independent, with a variance sampled independently from the uniform distribution $\mathcal U\left(1,\sqrt 2\right)$. 

To now test whether a particular DAG fits the data,
\cite{pfister18kernel} propose to use generalized additive model regression to find the residuals when regressing each node onto all its parents and to reject the DAG if the residuals are not jointly independent. If these are independent, we accept the causal structure. In this application, one is only interested in the relative $p$-values when performing the procedure for all possible DAGs with the correct number of nodes.

V-HSIC has the best performance in \citep{pfister18kernel}, so we only compare against V-HSIC; it is also the only other approach which allows testing joint independence of more than two components. Figure~\ref{fig:sim-dag} shows how often N-MHSIC and V-HSIC identify the correct DAG in 100 samples. V-HSIC has higher power than N-MHSIC and more often identifies the correct DAG for small sample sizes. However, as the r.h.s.\ of Figure~\ref{fig:sim-dag} shows, the proposed algorithm runs even for $n' = 8\sqrt n$ and $n=1500$ twice as fast as V-HSIC while producing the same result quality. Due to their different runtime complexities, the gap in runtime widens further with increasing sample size.

\begin{figure}
  \centering
  \includegraphics[width=0.99\linewidth]{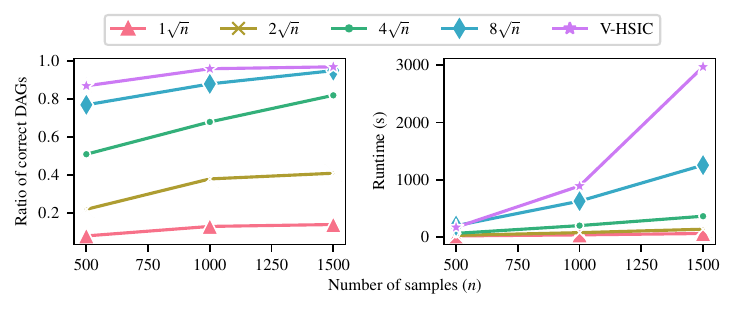}
  \caption{Ratio of correctly identified DAGs with $4$ nodes.}
  \label{fig:sim-dag}
\end{figure}

\subsection{Real-World Data} \label{subsec:real-problems}
This section is dedicated to benchmarks on real-world data.

\textbf{Million Song Data.} The Million Song Data \citep{bertin11million} contains approximately 500,000 songs. Each has 90 features ($X$) together with its year of release, which ranges from 1922 to 2011 ($Y$). The algorithms must detect the dependence between the features and the year of release. To approximate the power, we draw $100$ independent samples of the whole data set. Figure~\ref{fig:power-msd} shows the results, for level $\alpha=0.01$; the different ranges of $n$ highlight the asymptotic runtime gains. In contrast to a similar experiment of \cite{jitkrittum17adaptive2}, we use a permutation approach for all two-sample tests and increase the number of Nyström samples (random Fourier features) as a function of $n$, obtaining higher power throughout. The problem is sufficiently challenging, so that we set the number of Nyström samples to $8 \sqrt n$ for N-MHSIC. V-HSIC and NFSIC achieve maximum power from $n=500$. N-MHSIC features similar runtime and power as the existing HSIC approximations N-HSIC and RFF-HSIC but can handle more than two components. The runtime plot illustrates that the lower asymptotic complexity of N-MHSIC compared to V-HSIC also holds in practice.

\begin{figure}
  \centering
  \includegraphics[width=\linewidth]{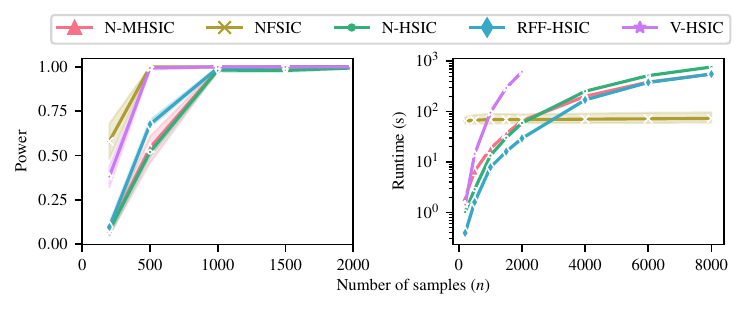}
  \caption{Test power vs.\ runtime on the Million Song Data.}
  \label{fig:power-msd}
\end{figure}

\paragraph{Weather Causal Discovery.} \label{sec:weather-causal} Here, we aim to infer the correct causality DAG from real-world data, namely the data set of \cite{mooij16distinguishing} which contains $349$ measurements consisting of altitude, temperature and sunshine. The goal is to infer the most plausible DAG with three nodes $(d=3)$ out of the $25$ possible DAGs ($3^3-2 = 25$; two graphs have a cycle). We assume the structural equations discussed before. Figure~\ref{fig:real-world-dag} shows the $p$-values with the estimated DAG (with index $25$) having the largest $p$-value. Again, we compare our results to V-HSIC  and find that both successfully identify the most plausible DAG \citep{pfister18kernel}.

These experiments demonstrate the efficiency of the proposed Nyström $M$-HSIC method.
\begin{figure}
    \centering
    \begin{subfigure}{0.59\linewidth}
         \centering
         \includegraphics[width=\linewidth]{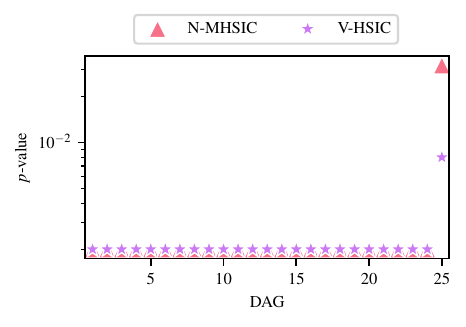}
        \end{subfigure}%
     ~
     \begin{subfigure}[c]{0.39\linewidth}
       \centering
       \vspace{-9.5em}
       \includegraphics[width=\linewidth]{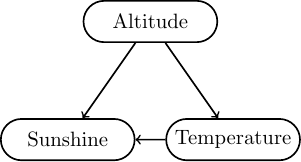}
    \end{subfigure}
    \caption{Testing for joint independence on the residuals of DAGs with three nodes (left) and the DAG with the largest $p$-value (right). The $p$-values agree on DAGs $1$ to $24$.}
    \label{fig:real-world-dag}
\end{figure}

\onecolumn
\appendix
\section{Appendix}
Section~\ref{sec:external-theorem} contains two external theorems and lemmas that we use. Section~\ref{sec:proofs} is about  our proofs.

\subsection{External Theorems and Lemmas}
\label{sec:external-theorem}
In this section two theorems and lemmas are recalled for self-completeness, Theorem~\ref{thm:rudi} is about bounding the error of Nyström mean embeddings \citep[Theorem 4.1]{chatalic22nystrom},  Theorem~\ref{thm:hoeffding} is a well-known result \citep[Section~5.6, Theorem~A]{serfling80approximation} for bounding the deviation of U-statistics. Lemma~\ref{lemma:conn-u-v} is about connection between U- and V-statistics. Lemma~\ref{lemma:markov} recalls Markov's inequality.

\begin{theorem}[Bound on mean embeddings]
  \label{thm:rudi} Let $\mathcal X$ be a locally compact second-countable topological space, $X$ a random variable supported on $\mathcal X$ with Borel probability measure $\P {}$, and let $\mathcal H_k$ be a RKHS on $\mathcal X$ with kernel $k : \X \times \X \to \R$, and feature map $\phi_k$. 
  Assume that there exists a constant $K \in (0,\infty)$ such that $\sup_{x \in \mathcal X}\sqrt{k(x,x)} \leq K$. Let $C_k=\E\left[\phi_k(X) \otimes \phi_k(X) \right]$. Furthermore, assume that the data points $\hat \P_{n} = \{x_1,\dots,x_n\}$ are drawn i.i.d.\ from the distribution $\P{}$ and that $n' \leq n$ subsamples $\tilde\P_{n'} = \{\tilde x_1, \dots, \tilde x_{n'}\}$ are drawn uniformly with replacement from the dataset $\hat\P_n$. Then for any $\delta \in (0,1)$ it holds that 
  \begin{align*}
  \norm{\mu_{k}\left(\P{}\right) - \mu_{k}\left(\tilde \P_{n'}\right) }{\H_k} \leq \frac{c_1}{ \sqrt{n}} + \frac{c_2}{n'} + \frac{c_3\sqrt{\log(n'/\delta)}}{n'}\sqrt{\mathcal N_{X}\left(\frac{12K^2\log(n'/\delta)}{n'}\right)},
  \end{align*}
  with probability at least $1-\delta$ provided that
  \begin{align*}
    n' \geq \max\left(67,12K^2\opnorm{C_k}^{-1}\right)\log\left(\frac{n'}{\delta}\right),
  \end{align*}
  where $c_1 = 2K\sqrt{2\log(6/\delta)}$, $c_2 = 4\sqrt 3 K \log(12/\delta)$, and $c_3 = 12\sqrt{3\log(12/\delta)}K$.
\end{theorem}

Recall that a U-statistic is the average of a (symmetric) core function $h=h(x_1,\dots,x_m)$ over the observations $X_1,\dots,X_n \sim \P$ ($n\ge m$)  with form
\begin{align}
\label{eq:def-u-stat}
  U_n = U(X_1,\dots,X_m) = \frac{1}{\binom{n}{m}}\sum_ch(X_{i_1},\dots,X_{i_m}),
\end{align}
where $c$ is the set of the $\binom n m $ combinations of $m$ distinct elements $\{i_1,\dots,i_m\}$ from $\{1,\dots,n\}$. $U_n$ is an unbiased estimator of $\theta = \theta(\P) = \E_\P[h(X_1,\dots,X_m)]$. 

\begin{theorem}[Hoeffding's inequality for U-statistics]
  \label{thm:hoeffding}Let $h = h(x_1,\dots,x_m)$ be a core function for $\theta = \theta(\P)=\E_{\P}\left[h(X_1,\dots,X_m)\right]$ with $a \leq h(x_1,\dots,x_m) \leq b$.  Then, for any $u> 0$ and $n\geq m$,
  \begin{align*}
    \P(U_n - \theta \geq u) \leq \exp\left(-\frac{2nu^2}{m(b-a)^2}\right).
  \end{align*}
\end{theorem}

Similar to \eqref{eq:def-u-stat} one can consider an alternative (slightly biased) estimator of $\theta$, which is called V-statistic:
\begin{align}
\label{eq:def-v-stat}
  V_n = V(X_1,\dots,X_m) = \frac{1}{{n}^{m}}\sum_{(i_1,\dots,i_m) \in T_m(n)} h(X_{i_1},\dots,X_{i_m}),
\end{align}
where $T_m(n)$ is the $m$-fold Cartesian product of the set $[n]$.

There is a close relation between U- and V-statistics, as it is made explicit by the following lemma  \citep[Lemma,~Section~5.7.3]{serfling80approximation}.

\begin{lemma}[Connection between U- and V-statistics]
\label{lemma:conn-u-v} Let $\P$ be a probability measure on a metric space $\X$. Let $(X_i)_{i\in[n]}\stackrel{\text{i.i.d.}}{\sim} \P$. Let $m$ denote any element of $[n]$. Let $h$ be a core function satisfying 
    $\E\left[|h(X_1,\dots,X_m)|^r\right] < \infty$ with some $r \in \mathbb Z_+$.  Let $U_n$ and $V_n$ denote the U and V-statistic associated to $h$ as defined in \eqref{eq:def-u-stat} and \eqref{eq:def-v-stat}, respectively.
Then it holds that
\begin{align*}
    \E\left[\left|U_n-V_n\right|^r\right] = \O\left(n^{-r}\right).
\end{align*}
\end{lemma}

\begin{lemma}[Markov inequality] \label{lemma:markov}
For a real-valued random variable $X$ with probability distribution $\P$ and $a > 0$, it holds that
\begin{align*}
    \P\left(|X| \ge a \right) \le \frac{\E\left(|X|\right)}{a}.
\end{align*}
\end{lemma}

\subsection{Proofs} \label{sec:proofs}
This section is dedicated to proofs. Lemma~\ref{thm:nystroem-hsic} is derived in Section~\ref{sec:mult-empir-nystr}. Proposition~\ref{thm:error-nystrom-hsic} is proved in Section~\ref{sec:error-nystrom-hsic} relying on two lemmas shown in Section~\ref{sec:technical-lemmas}. Lemma~\ref{lemma:deviation} is proved in Section~\ref{sec:proof-lemma-refl}, with an auxiliary result in Section~\ref{sec:u-stat-deviation}.

\subsubsection{Proof of Lemma \ref{thm:nystroem-hsic}}
\label{sec:mult-empir-nystr}
Let $\kmeP = \sum_{i=1}^{n'} \alpha_k^i \otimes_{m=1}^M \phi_{k_m}(\tilde x_m^i) $, and let $ \kmePm = \sum_{i=1}^{n'}\alpha_{k_m}^i\phi_{k_m}(\tilde x_m^i) $ for $m\in[M]$. We write
\begin{align*}
  \HSIC_{k,\text N}^2\left(\hat\P_{n}\right) &=\norm{ \kmeP- \otimes_{m=1}^M \kmePm}{\H_k}^2\\ 
                                               &= \underbrace{\norm{\kmeP}{\H_k}^2}_{=:A} -
                                                 2\cdot \underbrace{\left \langle \kmeP , \otimes_{m=1}^M\kmePm \right \rangle_{\H_k}}_{=: C} +
                                                 \underbrace{\norm{\otimes_{m=1}^M\kmePm}{\H_k}^2}_{=: B},
\end{align*}
and continue term-by-term. Using the definition of the tensor product, we have for term $A$ that
\begin{align*}
  A &= \left\langle \kmeP, \kmeP \right\rangle_{\H_k} 
    =  \sum_{i=1}^{n'}\sum_{j=1}^{n'} \alpha_k^i \alpha_k^j \left\langle \otimes_{m=1}^M \phi_{k_m}(\tilde x_m^i), \otimes_{m=1}^M \phi_{k_m}(\tilde x_m^j)\right\rangle_{\H_k} \\
    &=   \sum_{i=1}^{n'}\sum_{j=1}^{n'} \alpha_k^i \alpha_k^j \prod_{m=1}^Mk_m(\tilde x_m^i,\tilde x_m^j) 
    = \bm \alpha_k\tran\left( \hadamard_{m=1}^M \mathbf{K}_{k_m,n'n'}\right) \bm \alpha_k.
\end{align*}
Similarly, we obtain for term $B$ that
\begin{align*}
  B &= \left\langle \otimes_{m=1}^M\kmePm, \otimes_{m=1}^M \kmePm \right\rangle_{\H_k} \\
    &=  \left\langle \otimes_{m=1}^M \sum_{i^{(m)}=1}^{n'}\alpha_{k_m}^{i^{(m)}}\phi_{k_m}\left(\tilde x_m^{i^{(m)}}\right), \otimes_{m=1}^M \sum_{j^{(m)}=1}^{n'}\alpha_{k_m}^{j^{(m)}}\phi_{k_m}\left(\tilde x_m^{j^{(m)}}\right) \right\rangle_{\H_k} \\
    &\stackrel{(*)}{=} \prod_{m=1}^M\sum_{i^{(m)}=1}^{n'}\sum_{j^{(m)}=1}^{n'} \alpha_{k_m}^{i^{(m)}}\alpha_{k_m}^{j^{(m)}} k_m\left(\tilde x_m^{i^{(m)}},\tilde x_m^{j^{(m)}}\right) 
  =  \prod_{m=1}^M \bm \alpha_{k_m}\tran \mathbf{K}_{k_m,n'n'} \bm\alpha_{k_m},
\end{align*}
where in $(*)$ we used \eqref{eq:tensor:inner-product}, the linearity of the inner product, and the reproducing property.

Last, we express term $C$ as
\begin{align*}
 C &= \left\langle \sum_{i=1}^{n'} \alpha_k^i \otimes_{m=1}^M \phi_{k_m}\left(\tilde x_m^i\right) ,\otimes_{m=1}^M \sum_{j^{(m)}=1}^{n'}\alpha_{k_m}^{j^{(m)}}\phi_{k_m}\left(\tilde x_m^{j^{(m)}}\right) \right \rangle_{\H_k} \\
&\stackrel{(a)}{=} \sum_{i=1}^{n'} \alpha_k^i \left\langle  \otimes_{m=1}^M \phi_{k_m}\left(\tilde x_m^i\right) ,\otimes_{m=1}^M \sum_{j^{(m)}=1}^{n'}\alpha_{k_m}^{j^{(m)}}\phi_{k_m}\left(\tilde x_m^{j^{(m)}}\right) \right \rangle_{\H_k} \\
&\stackrel{(b)}{=} \sum_{i=1}^{n'}\alpha_k^i \prod_{m\in[M]}  \left\langle   \phi_{k_m}\left(\tilde x_m^i\right) , \sum_{j^{(m)}=1}^{n'}\alpha_{k_m}^{j^{(m)}}\phi_{k_m}\left(\tilde x_m^{j^{(m)}}\right) \right \rangle_{\H_k} \\
&\stackrel{(c)}{=} \sum_{i=1}^{n'}  \alpha_k^i\prod_{m\in[M]} \sum_{j^{(m)}=1}^{n'}\alpha_{k_m}^{j^{(m)}} \left\langle   \phi_{k_m}\left(\tilde x_m^i\right) , \phi_{k_m}\left(\tilde x_m^{j^{(m)}}\right) \right \rangle_{\H_k} \\
&\stackrel{(d)}{=} \sum_{i=1}^{n'} \alpha_k^i\prod_{m\in[M]} \underbrace{\sum_{j^{(m)}=1}^{n'}\alpha_{k_m}^{j^{(m)}}  k_m\left(\tilde x_m^i,\tilde x_m^{j^{(m)}}\right )}_{\left(\b K_{k_m,n'n'}\right)_i\bm\alpha_{k_m}} 
   = \bm\alpha_k\tran \left ( \hadamard_{m=1}^M \mathbf{K}_{k_m,n'n'}\bm\alpha_{k_m} \right),
\end{align*}
where (a) follows from the linearity of the inner product, (b) holds by  \eqref{eq:tensor:inner-product}, (c) is implied by the linearity of the inner product, (d) is valid by the reproducing property, and we refer to the $i$-th row of $\b K_{k_m,n'n'}$ as $\left(\b K_{k_m,n'n'}\right)_i$.

Substituting terms $A, B$, and $C$ concludes the proof.

\subsubsection{Two Lemmas to the Proof of Proposition~\ref{thm:error-nystrom-hsic}}
\label{sec:technical-lemmas}

Our main result relies on two lemmas.

\begin{lemma}[Error bound for Nyström mean embedding of tensor product kernel]
  \label{lemma:nystrom-cross-cov}    
  Let  $X=(X_m)_{m=1}^M \in \mathcal X = \times_{m=1}^M \X_m$, $X\sim\P{}\in \mathcal M_1^+(\X{})$, and $(\X_m)_{m\in [M]}$ locally compact, second-countable topological spaces.  Let $k_m: \mathcal X_m \times \mathcal X_m \rightarrow \R$ be a bounded kernel, i.e.\ there exists  $a_{k_m} \in (0,\infty)$ such that $\sup_{x_m\in \X_m}\sqrt{k_m(x_m,x_m)} \leq a_{k_m}$ for $m \in [M]$. Let $a_k = \prod_{m=1}^M a_{k_m}$, $k=\otimes_{m=1}^M k_m$, $\H_k$ the RKHS associated to $k$, $\phi_k = \otimes_{m=1}^M \phi_{k_m}$, $C_k = \E_{}\left[\phi_{k}(X) \otimes \phi_{k}(X) \right]$, $n' \le n$, and $\tilde \P_{n'}$ defined according to \eqref{eq:Nystrom-samples}. Then for any $\delta \in (0,1)$ it holds that 
  \begin{align*}
    \norm{\mu_{k}\left(\P{}\right) - \mu_{k}\left(\tilde \P_{n'}\right)}{\H_k} \leq \frac{c_{k,1}}{ \sqrt{n}} + \frac{c_{k,2}}{n'} + \frac{c_{k,3}\sqrt{\log(n'/\delta)}}{n'}\sqrt{\mathcal N_{X}\left(\frac{12a_k^2\log(n'/\delta)}{n'}\right)},
  \end{align*}
   with probability at least $1-\delta$, provided that
  \begin{align*}
    n' \geq \max\left(67,12a_k^2\opnorm{C_k}^{-1}\right)\log\left(\frac{n'}{\delta}\right),
  \end{align*}
  where $c_{k,1}= 2a_k\sqrt{2\log(6/\delta)}$, $c_{k,2}=4\sqrt 3 a_k \log(12/\delta)$, and $c_{k,3}= 12\sqrt{3\log(12/\delta)}a_k$.
\end{lemma}

\begin{proof}
With $\X = \times_{m\in [M]} \X_m$, noticing that $\X$ is locally compact second-countable iff.\ $(\X_m)_{m\in [M]}$ are so \citep[Theorem~16.2(c), Theorem~18.6]{willard70general}, $\H_k = \tphs$, $\phi_k = \otimes_{m=1}^M \phi_{k_m} $, and $\sqrt{k(x,x)}= \prod_{m=1}^M\sqrt{k_m(x_m,x_m)} \leq a_k $, the statement is implied by Theorem~\ref{thm:rudi}.
\end{proof}

\begin{proof}[Proof of Lemma~\ref{lemma:decomposition}]
    
To simplify notation, let $\mu_{k_m} = \mu_{k_m}\left(\P_m\right)$, $\tilde \mu_{k_m} = \mu_{k_m}\left(\tilde\P_{m,n'}\right)$, $\H_k = \tphs$, and $d_{k_m} = \norm{\mu_{k_m}-\tilde \mu_{k_m}}{\H_{k_m}}$. The proof proceeds by induction on $M$: For $M=1$ the l.h.s. = r.h.s. = $\norm{\mu_{k_1}\left(\P_1\right) - \mu_{k_1}\left(\tilde\P_{1,n'}\right)}{\H_k} $ is satisfied, and
    we assume that the statement holds for $M = M-1$, to obtain
    \begin{eqnarray*}
      \lefteqn{\norm{\otimes_{m=1}^M \mu_{k_m} - \otimes_{m=1}^M \tilde\mu_{k_m}}{\H_k} = \big\|\otimes_{m=1}^M \mu_{k_m} - \otimes_{m=1}^{M-1} \mu_{k_m} \otimes \tilde\mu_{k_M} +  \otimes_{m=1}^{M-1} \mu_{k_m} \otimes \tilde\mu_{k_M}  - \otimes_{m=1}^M \tilde\mu_{k_m}\big\|_{\H_k}}  \\
    &&= \left\|\otimes_{m=1}^{M-1} \mu_{k_m} \otimes (\mu_{k_M} - \tilde\mu_{k_M}) \right. + \left. \left(\otimes_{m=1}^{M-1} \mu_{k_m}- \otimes_{m=1}^{M-1} \tilde\mu_{k_m}\right) \otimes \tilde\mu_{k_M} \right\|_{\H_k} \\
    && \stackrel{(a)}{\le} \left\|\otimes_{m=1}^{M-1} \mu_{k_m} \otimes (\mu_{k_M} - \tilde\mu_{k_M}) \right\|_{\H_k} + \left\|\left(\otimes_{m=1}^{M-1} \mu_{k_m}- \otimes_{m=1}^{M-1} \tilde\mu_{k_m}\right) \otimes \tilde\mu_{k_M} \right\|_{\H_k} \\
    &&\stackrel{(b)}{=} \left(\prod_{m \in [M-1]} \left\| \mu_{k_m}\right\|_{\H_{k_m}}\right) d_{k_M} + \left\|\otimes_{m=1}^{M-1} \mu_{k_m}- \otimes_{m=1}^{M-1} \tilde\mu_{k_m} \right\|_{\otimes_{m=1}^{M-1} \H_{k_m}}\left\|\tilde\mu_{k_M} \right\|_{\H_{k_M}} \\
    &&\stackrel{(c)}{\le} d_{k_M}\prod_{m \in [M-1]} a_{k_m}  + \left\|\otimes_{m=1}^{M-1} \mu_{k_m}- \otimes_{m=1}^{M-1} \tilde\mu_{k_m} \right\|_{\otimes_{m=1}^{M-1} \H_{k_m}}\left(a_{k_M} + d_{k_M}\right) \\
    &&\stackrel{(d)}{\le} d_{k_M}\prod_{m \in [M-1]} a_{k_m} + \left\{\prod_{m\in[M-1]}\left(a_{k_m} + d_{k_m}\right) - \prod_{m\in[M-1]}a_{k_m}\right\}\left(a_{k_M} + d_{k_M}\right) \\
    &&= d_{k_M}\prod_{m \in [M-1]} a_{k_m} + \prod_{m\in[M]}\left(a_{k_m} + d_{k_m}\right) - \prod_{m\in[M]}a_{k_m} - d_{k_M}\prod_{m \in [M-1]}a_{k_m} \\
    &&= \prod_{m\in[M]}\left(a_{k_m} + d_{k_m}\right) - \prod_{m\in[M]}a_{k_m},
    \end{eqnarray*}
    where (a) holds by the triangle inequality, (b) is implied by \eqref{eq:tensor:norm} and the definition of $d_{k_M}$, (c) follows from
    \begin{align}
    \left\| \mu_{k_m}\right\|_{\H_{k_m}} & = \left\| \int_{\X_m} k_m(\cdot,x_m) \d \P_m (x_m)\right\|_{\H_{k_m}} \stackrel{(e)}{\le} \int_{\X_m} \underbrace{\left\|k_m(\cdot,x_m) \right\|_{\H_{k_m}}}_{\stackrel{(f)}{=}\sqrt{k_m(x_m,x_m)} \stackrel{(g)}{\le} a_{k_m}} \d \P_m (x_m) \le a_{k_m}, \label{eq:ME:bound}\\
    \norm{\tilde\mu_{k_M}}{\H_{k_M}} & = \norm{\tilde\mu_{k_M} - \mu_{k_M} + \mu_{k_M}}{\H_{k_M}} \stackrel{(h)}{\le} \norm{\tilde\mu_{k_M} - \mu_{k_M}}{\H_{k_M}} + \norm{\mu_{k_M}}{\H_{k_M}} \stackrel{(i)}{\le} d_{k_M} + a_{k_M}, \nonumber
    \end{align}
    (d) is valid by the induction statement holding for $M-1$, 
    (e) is a property of Bochner integrals, (f) is implied by the reproducing property, (g) comes from the definition of $a_{k_m}$, the triangle inequality implies (h), (i) follows from \eqref{eq:ME:bound} and the definition of $d_{k_M}$.
  \end{proof}

\subsubsection{Proof of Proposition \ref{thm:error-nystrom-hsic}}
\label{sec:error-nystrom-hsic}

Let $k = \otimes_{m=1}^M k_m$, and let $\H_k = \otimes_{m=1}^M \H_{k_m}$. We note that $\X=\times_{m \in [M]} \X_m$ is locally compact second-countable as $(\X_m)_{m\in [M]}$ are so \citep[Theorem~16.2(c), Theorem~18.6]{willard70general}.

We decompose the error of the Nyström approximation as 
  \begin{align*}
    \left|\mathrm{HSIC}_k(\P{}) - \HSIC_{k,\text N}\left(\hat \P_{n}\right) \right| 
    &= \left|\norm{\mu_{k}(\P{})- \otimes_{m=1}^M \mu_{k_m}( \P_{m}) }{\H_k} - \norm{ \kmeP- \otimes_{m=1}^M \kmePm}{\H_k} \right| \\
     &\stackrel{(a)}{\le} \norm{\mu_{k}(\P{})- \otimes_{m=1}^M \mu_{k_m}( \P_{m})  - \kmeP +  \otimes_{m=1}^M \kmePm}{\H_k} \\
    &\stackrel{(b)}{\le} \underbrace{\norm{\mu_{k}(\P{})- \kmeP }{\H_{k}}}_{t_1} +\underbrace{\norm{ \otimes_{m=1}^M \mu_{k_m}( \P_{m})  -  \otimes_{m=1}^M \kmePm}{\H_{k}}}_{t_2}, 
  \end{align*}
  where (a) holds by the reverse triangle inequality, and (b) follows from the triangle inequality. 
  
  \tb{First term ($t_1$)}:  One can bound the error of the first term by Lemma~\ref{lemma:nystrom-cross-cov}; in other words, for any $\delta \in (0,1)$ with probability at least $(1-\delta)$ it holds that
  \begin{align*}
      \norm{\mu_{k}(\P{})- \kmeP }{\H_{k}} \le \frac{c_{k,1}}{ \sqrt{n}} + \frac{c_{k,2}}{n'} + \frac{c_{k,3}\sqrt{\log(n'/\delta)}}{n'}\sqrt{\mathcal N_{X_m}\left(\frac{12a_{k_m}^2\log(n'/\delta)}{n'}\right)}
  \end{align*}
   provided that $n' \geq \max\left(67,12a_{k}^2\opnorm{C_{k}}^{-1}\right)\log\left(\frac{n'}{\delta}\right)$,
  with the constants $c_{k,1} =  2a_{k}\sqrt{2\log(6/\delta)}$, $c_{k,2}=4\sqrt 3 a_{k} \log(12/\delta)$, $c_{k,3} =  12\sqrt{3\log(12/\delta)}a_{k}$.
  
  \tb{Second term} ($t_2$): Applying Lemma~\ref{lemma:decomposition} to the second term gives
  \begin{align*}
     \norm{ \otimes_{m=1}^M \mu_{k_m}( \P_{m})  -  \otimes_{m=1}^M \kmePm}{\H_{k}} &\le  \prod_{m\in[M]}\left(a_{k_m}+\norm{\mu_{k_m}\left(\P_m\right) -  \mu_{k_m}\left(\tilde \P_{m,n'}\right)}{\H_{k_m}}\right) - \prod_{m\in[M]}a_{k_m}.
  \end{align*}
We now bound the error of each of the $M$ factors by Theorem~\ref{thm:rudi}, i.e., for fixed $m \in [M]$; particularly we get that for any $\delta\in(0,1)$ with probability at least $1-\delta$ 
\begin{align*}
      \norm{\mu_{k_m}\left(\P_m\right) - \kmePm}{\H_{k_m}} &\le \frac{c_{k_m,1}}{ \sqrt{n}} + \frac{c_{k_m,2}}{n'} + \frac{c_{k_m,3}\sqrt{\log(n'/\delta)}}{n'}\sqrt{\mathcal N_{X_m}\left(\frac{12a_{k_m}^2\log(n'/\delta)}{n'}\right)}, \text{ hence} \\
      a_{k_m} + \norm{\mu_{k_m}\left(\P_m\right) - \kmePm}{\H_{k_m}} &\le a_{k_m} + \frac{c_{k_m,1}}{ \sqrt{n}} + \frac{c_{k_m,2}}{n'} + \frac{c_{k_m,3}\sqrt{\log(n'/\delta)}}{n'}\sqrt{\mathcal N_{X_m}\left(\frac{12a_{k_m}^2\log(n'/\delta)}{n'}\right)},
\end{align*}
and by union bound that their product is for any $\delta \in (0,\frac{1}{M})$ with probability at least $1-M\delta$ 
\begin{eqnarray*}
      \lefteqn{\prod_{m\in[M]}\left[a_{k_m} + \norm{\mu_{k_m}\left(\P_m\right) - \kmePm}{\H_{k_m}} \right]  \le}\\
      && \le \prod_{m\in[M]}\Bigg[a_{k_m} + \frac{c_{k_m,1}}{ \sqrt{n}} + \frac{c_{k_m,2}}{n'} + \frac{c_{k_m,3}\sqrt{\log(n'/\delta)}}{n'}\sqrt{\mathcal N_{X_m}\left(\frac{12a_{k_m}^2\log(n'/\delta)}{n'}\right)}\Bigg],
\end{eqnarray*}
\begin{eqnarray*}      
      \lefteqn{\prod_{m\in[M]}\left[a_{k_m} + \norm{\mu_{k_m}\left(\P_m\right) - \kmePm}{\H_{k_m}} \right] -\prod_{m\in[M]}a_{k_m} \le}\\
      && \le \prod_{m\in[M]}\Bigg[a_{k_m} + \frac{c_{k_m,1}}{ \sqrt{n}} + \frac{c_{k_m,2}}{n'} + \frac{c_{k_m,3}\sqrt{\log(n'/\delta)}}{n'}\sqrt{\mathcal N_{X_m}\left(\frac{12a_{k_m}^2\log(n'/\delta)}{n'}\right)}\Bigg] -\prod_{m\in[M]}a_{k_m},
\end{eqnarray*}
  provided that $
    n' \geq \max\left(67,12a_{k_m}^2\opnorm{C_{k_m}}^{-1}\right)\log\left(\frac{n'}{\delta}\right)$
    for all $m\in [M]$, 
  with  $C_{k_m} = \E\left[\phi_{k_m}(X_m) \otimes \phi_{k_m}(X_m)\right]$ and constants $c_{k_m,1} =  2a_{k_m}\sqrt{2\log(6/\delta)}$, $c_{k_m,2}=4\sqrt 3 a_{k_m} \log(12/\delta)$, $c_{k_m,3} =  12\sqrt{3\log(12/\delta)}a_{k_m}$, with $m\in[M]$.
  
  Combining the $M+1$ terms by union bound yields the stated result.

\subsubsection{Lemma to the Proof of Lemma~\ref{lemma:deviation}}
\label{sec:u-stat-deviation}
\begin{lemma}[Deviation bound for U-statistics based HSIC estimator]
\label{lemma:deviation-u-stat}It holds that
\begin{align*}
    \left|\HSIC_{k,u}^2\left(\hat\P_n\right)-\HSIC_{k}^2\left(\P\right)\right| = \O_{\text{P}}\left(\frac{1}{\sqrt{n}}\right),
\end{align*}
where $\HSIC_{k,u}^2$ is the U-statistic based estimator of $\HSIC_k^2$.
\end{lemma}
\begin{proof}
We show that \eqref{eq:def-hsic} can be expressed as a sum of U-statistics and then bound the terms individually. First, square \eqref{eq:def-hsic} to obtain
\begin{align*}
  \HSIC_k^2(\P) &= \underbrace{\E_{(x_1,\dots,x_M),(x'_1,\dots,x_M')\sim\P}\left[\prod_{m\in[M]}k_m\left(x_m,x_m'\right)\right]}_A +
                  \underbrace{\E_{x_1,x_1'\sim\P_1,\dots,x_M,x_M'\sim \P_M}\left[\prod_{m\in[M]}k_m\left(x_m,x_m'\right)\right]}_B \\
                &\quad\quad- 2\underbrace{\E_{(x_1,\dots,x_M)\sim\P, x_1'\sim\P_1, \dots, x_M'\sim\P_M}\left[\prod_{m\in[M]}k_m(x_m,x_m')\right]}_C,
\end{align*}
where $A$, $B$, and $C$ can be estimated by U-statistics $A'_n$, $B'_n$, and $C'_n$, respectively. Let $\HSIC_{k,u}^2\left(\hat \P_n\right) = A'_n + B'_n - 2C'_n$, and split $t$ as $\alpha t +\beta t + (1-\alpha-\beta)t$, with $\alpha, \beta > 0$ and $\alpha + \beta < 1$. One obtains
\begin{align*}
  P\left(\left|\HSIC_k^2(\P) - \HSIC_{k,u}^2\left(\hat\P_n\right)\right | \ge t \right) \leq P\left(\left|A-A'_n\right| \ge \alpha t\right) \hspace{-0.1cm}+ \hspace{-0.1cm} P\left(\left|B-B'_n\right| \ge \beta t\right) \hspace{-0.1cm}+\hspace{-0.1cm} P\left(2\left|C-C'_n\right| \ge (1-\alpha-\beta) t\right).
\end{align*}
Doubling and rewriting Theorem~\ref{thm:hoeffding}, we have that for U-statistics and any $\delta \in (0,1)$
\begin{align*}
  \P\left(\left|U_n-\theta\right| \ge \sqrt{\frac{m(b-a)^2\ln(\frac 2 \delta)}{2n}}\right) \leq \delta.
\end{align*}
Now, choosing the $\left(\theta, U_n, u\right)$ triplet to be $\left(A,A_n',\alpha t\right)$, $\left(B,B_n', \beta t \right)$, $\left(C,C_n', \frac{(1-\alpha-\beta)t}{2}\right)$, respectively, setting $m=2M$, and observing that $a \le k(x,y) \leq b$ as $k$ is bounded, we obtain that $|A'_n-A|\sqrt{n}$, $|B'_n-B|\sqrt{n}$, and $|C'_n-C|\sqrt{n}$ are bounded in probability and so is their sum.
\end{proof}
  
  \subsubsection{Proof of Lemma~\ref{lemma:deviation}}
  \label{sec:proof-lemma-refl}

We consider the decomposition
\begin{align}
    \left|\HSIC_k^2\left(\hat\P_n\right) - \HSIC_k^2\left(\P\right)\right| \le \underbrace{\left|\HSIC_k^2\left(\hat\P_n\right) - \HSIC_{k,u}^2\left(\hat\P_n\right) \right|}_{t_1} +\underbrace{\left|\HSIC_{k,u}^2\left(\hat\P_n\right) - \HSIC^2_k\left(\P\right)\right|}_{t_2}, \label{eq:decompo-1}
\end{align}
by using the triangle inequality,  where $\HSIC_{k,u}$ is the U-statistic based HSIC estimator.

\tb{Second term ($t_2$)}: Lemma~\ref{lemma:deviation-u-stat} establishes that $t_2=\O_{\text{P}}\left(\frac{1}{\sqrt{n}}\right)$. 

\tb{First term ($t_1$)}: To bound $t_1$, first, by  Markov's inequality (Lemma~\ref{lemma:markov}) observe that 
    \begin{align}
    \P\left(\left|\HSIC_k^2\left(\hat\P_n\right) - \HSIC_{k,u}^2\left(\hat\P_n\right) \right| \ge a \right) &\le \underbrace{\frac{\E\left(\left|\HSIC_k^2\left(\hat\P_n\right) - \HSIC_{k,u}^2\left(\hat\P_n\right) \right|\right)}{a}}_{=: \epsilon}, \nonumber\\
    \P\left(\left|\HSIC_k^2\left(\hat\P_n\right) - \HSIC_{k,u}^2\left(\hat\P_n\right) \right| \ge \frac{\E\left(\left|\HSIC_k^2\left(\hat\P_n\right) - \HSIC_{k,u}^2\left(\hat\P_n\right) \right|\right)}{\epsilon} \right)&\le \epsilon, \nonumber\\
    \P\left(\left|\HSIC_k^2\left(\hat\P_n\right) - \HSIC_{k,u}^2\left(\hat\P_n\right) \right| < \frac{\E\left(\left|\HSIC_k^2\left(\hat\P_n\right) - \HSIC_{k,u}^2\left(\hat\P_n\right) \right|\right)}{\epsilon} \right) &\ge 1- \epsilon, \nonumber\\
    \P\left(\left|\HSIC_k^2\left(\hat\P_n\right) - \HSIC_{k,u}^2\left(\hat\P_n\right) \right| < \frac{C}{n\epsilon} \right) &\stackrel{(*)}{\ge} 1- \epsilon, \nonumber\\
    \P\left(\left|\HSIC_k^2\left(\hat\P_n\right) - \HSIC_{k,u}^2\left(\hat\P_n\right) \right| \ge \frac{C}{n\epsilon} \right) &\le  \epsilon, \label{eq:HSIC-1/n}
\end{align}
for constant $C >0$ and $n$ large enough, where $(*)$ follows from Lemma~\ref{lemma:conn-u-v} (with $r=1$). \eqref{eq:HSIC-1/n} implies that 
\begin{align*}
    \left|\HSIC_k^2\left(\hat\P_n\right) - \HSIC_{k,u}^2\left(\hat\P_n\right) \right| = \O_{\text{P}}\left(\frac{1}{n}\right).
\end{align*}
\tb{Combining the terms ($t_1 + t_2$)}: Combining the obtained results for the two terms, one gets that
\begin{align}
\left|\HSIC_k^2\left(\hat\P_n\right) - \HSIC_k^2\left(\P\right)\right| &\stackrel{\eqref{eq:decompo-1}}{\le}     \left|\HSIC_k^2\left(\hat\P_n\right) - \HSIC_{k,u}^2\left(\hat\P_n\right) \right| +\left|\HSIC_{k,u}^2\left(\hat\P_n\right) - \HSIC_k^2\left(\P\right)\right| \nonumber\\
& =  \O_{\text{P}}\left(\frac{1}{n}\right) + \O_{\text{P}}\left(\frac{1}{\sqrt n}\right) = \O_{\text{P}}\left(\frac{1}{\sqrt n}\right). \label{eq:decompo-2}
\end{align}
Hence
\begin{align*}
    \O_{\text{P}}\left(\frac{1}{\sqrt n}\right) & \stackrel{\eqref{eq:decompo-2}}{\ge}  \left|\HSIC_k^2\left(\hat\P_n\right) - \HSIC_k^2\left(\P\right)\right| = \left|\HSIC_k\left(\hat\P_n\right) - \HSIC_k\left(\P\right)\right|
     \Big|\underbrace{\HSIC_k\left(\hat\P_n\right)}_{\stackrel{\eqref{eq:emp-hsic}}{\ge} 0} + \underbrace{\HSIC_k\left(\P\right)}_{\ge 0}\Big| \\
    &\ge \left|\HSIC_k\left(\hat\P_n\right) - \HSIC_k\left(\P\right)\right|\HSIC_k\left(\P\right),
\end{align*}
which by dividing with the constant $\HSIC_k\left(\P\right)>0$ implies the statement.

\begin{acknowledgements}
This work was supported by the German Research Foundation
(DFG) Research Training Group GRK 2153: Energy Status Data –-
Informatics Methods for its Collection, Analysis and Exploitation. The major part of this work was carried out while Florian Kalinke was a research associate at the Department of Statistics, London School of Economics.
\end{acknowledgements}

\bibliography{BIB/collected_zoltan,BIB/publications}

\begin{thebibliography}{55}
\providecommand{\natexlab}[1]{#1}
\providecommand{\url}[1]{\texttt{#1}}
\expandafter\ifx\csname urlstyle\endcsname\relax
  \providecommand{\doi}[1]{doi: #1}\else
  \providecommand{\doi}{doi: \begingroup \urlstyle{rm}\Url}\fi

\bibitem[Albert et~al.(2022)Albert, Laurent, Marrel, and
  Meynaoui]{albert22adaptive}
M\'{e}lisande Albert, B\'{e}atrice Laurent, Amandine Marrel, and Anouar
  Meynaoui.
\newblock Adaptive test of independence based on {HSIC} measures.
\newblock \emph{The Annals of Statistics}, 50\penalty0 (2):\penalty0 858--879,
  2022.

\bibitem[Aronszajn(1950)]{aronszajn50theory}
Nachman Aronszajn.
\newblock Theory of reproducing kernels.
\newblock \emph{Transactions of the American Mathematical Society},
  68:\penalty0 337--404, 1950.

\bibitem[Berlinet and Thomas-Agnan(2004)]{berlinet04reproducing}
Alain Berlinet and Christine Thomas-Agnan.
\newblock \emph{Reproducing Kernel Hilbert Spaces in Probability and
  Statistics}.
\newblock Kluwer, 2004.

\bibitem[Bertin-Mahieux et~al.(2011)Bertin-Mahieux, Ellis, Whitman, and
  Lamere]{bertin11million}
Thierry Bertin-Mahieux, Daniel P.~W. Ellis, Brian Whitman, and Paul Lamere.
\newblock The million song dataset.
\newblock In \emph{International Society for Music Information Retrieval
  Conference (ISMIR)}, pages 591--596, 2011.

\bibitem[Borgwardt et~al.(2020)Borgwardt, Ghisu, Llinares-L{\'o}pez, O'Bray,
  and Riec]{borgwardt20graph}
Karsten Borgwardt, Elisabetta Ghisu, Felipe Llinares-L{\'o}pez, Leslie O'Bray,
  and Bastian Riec.
\newblock Graph kernels: State-of-the-art and future challenges.
\newblock \emph{Foundations and Trends in Machine Learning}, 13\penalty0
  (5-6):\penalty0 531--712, 2020.

\bibitem[Bouche et~al.(2023)Bouche, Flamary, d’Alch{\'e} Buc, Plougonven,
  Clausel, Badosa, and Drobinski]{bouche23wind}
Dimitri Bouche, R{\'e}mi Flamary, Florence d’Alch{\'e} Buc, Riwal Plougonven,
  Marianne Clausel, Jordi Badosa, and Philippe Drobinski.
\newblock Wind power predictions from nowcasts to 4-hour forecasts: a learning
  approach with variable selection.
\newblock \emph{Renewable Energy}, 2023.

\bibitem[Camps-Valls et~al.(2010)Camps-Valls, Mooij, and
  Sch{\"{o}}lkopf]{camps10remote}
Gustavo Camps-Valls, Joris~M. Mooij, and Bernhard Sch{\"{o}}lkopf.
\newblock Remote sensing feature selection by kernel dependence measures.
\newblock \emph{{IEEE} {G}eoscience and {R}emote {S}ensing {L}etters},
  7\penalty0 (3):\penalty0 587--591, 2010.

\bibitem[Chakraborty and Zhang(2019)]{chakraborty19distance}
Shubhadeep Chakraborty and Xianyang Zhang.
\newblock Distance metrics for measuring joint dependence with application to
  causal inference.
\newblock \emph{Journal of the American Statistical Association}, 114\penalty0
  (528):\penalty0 1638--1650, 2019.

\bibitem[Chatalic et~al.(2022)Chatalic, Schreuder, Rudi, and
  Rosasco]{chatalic22nystrom}
Antoine Chatalic, Nicolas Schreuder, Alessandro Rudi, and Lorenzo Rosasco.
\newblock Nystr{\"o}m kernel mean embeddings.
\newblock In \emph{International Conference on Machine Learning (ICML)}, pages
  3006--3024, 2022.

\bibitem[Chwialkowski et~al.(2015)Chwialkowski, Ramdas, Sejdinovic, and
  Gretton]{chwialkowski15fast}
Kacper Chwialkowski, Aaditya Ramdas, Dino Sejdinovic, and Arthur Gretton.
\newblock Fast two-sample testing with analytic representations of probability
  measures.
\newblock In \emph{Advances in Neural Information Processing Systems (NIPS)},
  pages 1972--1980, 2015.

\bibitem[Climente-Gonz{\'a}lez et~al.(2019)Climente-Gonz{\'a}lez, Azencott,
  Kaski, and Yamada]{gonzalez19block}
Héctor Climente-Gonz{\'a}lez, Chlo{\'e}-Agathe Azencott, Samuel Kaski, and
  Makoto Yamada.
\newblock Block {HSIC} {L}asso: model-free biomarker detection for ultra-high
  dimensional data.
\newblock \emph{Bioinformatics}, 35\penalty0 (14):\penalty0 i427--i435, 2019.

\bibitem[Fukumizu et~al.(2008)Fukumizu, Gretton, Sun, and
  Sch{\"o}lkopf]{fukumizu08kernel}
Kenji Fukumizu, Arthur Gretton, Xiaohai Sun, and Bernhard Sch{\"o}lkopf.
\newblock Kernel measures of conditional dependence.
\newblock In \emph{Advances in Neural Information Processing Systems (NIPS)},
  pages 498--496, 2008.

\bibitem[G{\"a}rtner et~al.(2002)G{\"a}rtner, Flach, Kowalczyk, and
  Smola]{gartner02multi}
Thomas G{\"a}rtner, Peter Flach, Adam Kowalczyk, and Alexander Smola.
\newblock Multi-instance kernels.
\newblock In \emph{International Conference on Machine Learning (ICML)}, pages
  179--186, 2002.

\bibitem[Gretton et~al.(2005)Gretton, Bousquet, Smola, and
  Sch{\"o}lkopf]{gretton05measuring}
Arthur Gretton, Olivier Bousquet, Alex Smola, and Bernhard Sch{\"o}lkopf.
\newblock Measuring statistical dependence with {H}ilbert-{S}chmidt norms.
\newblock In \emph{Algorithmic Learning Theory (ALT)}, pages 63--78, 2005.

\bibitem[Gretton et~al.(2008)Gretton, Fukumizu, Teo, Song, Sch{\"o}lkopf, and
  Smola]{gretton08kernel}
Arthur Gretton, Kenji Fukumizu, Choon~Hui Teo, Le~Song, Bernhard Sch{\"o}lkopf,
  and Alexander Smola.
\newblock A kernel statistical test of independence.
\newblock In \emph{Advances in Neural Information Processing Systems (NIPS)},
  pages 585--592, 2008.

\bibitem[Gretton et~al.(2012)Gretton, Borgwardt, Rasch, Sch{\"o}lkopf, and
  Smola]{gretton12kernel}
Arthur Gretton, Karsten Borgwardt, Malte Rasch, Bernhard Sch{\"o}lkopf, and
  Alexander Smola.
\newblock A kernel two-sample test.
\newblock \emph{Journal of Machine Learning Research}, pages 723--773, 2012.

\bibitem[Guevara et~al.(2017)Guevara, Hirata, and Canu]{guevara17cross}
Jorge Guevara, Roberto Hirata, and St{\'e}phane Canu.
\newblock Cross product kernels for fuzzy set similarity.
\newblock In \emph{International Conference on Fuzzy Systems (FUZZ-IEEE)},
  pages 1--6, 2017.

\bibitem[Hagrass et~al.(2022)Hagrass, Sriperumbudur, and Li]{hagrass22spectral}
Omar Hagrass, Bharath~K Sriperumbudur, and Bing Li.
\newblock Spectral regularized kernel two-sample tests.
\newblock Technical report, 2022.
\newblock (\url{https://arxiv.org/abs/2212.09201}).

\bibitem[Haussler(1999)]{haussler99convolution}
David Haussler.
\newblock Convolution kernels on discrete structures.
\newblock Technical report, University of California at Santa Cruz, 1999.
\newblock
  (\url{http://cbse.soe.ucsc.edu/sites/default/files/convolutions.pdf}).

\bibitem[Huang et~al.(2022)Huang, Deb, and Sen]{huang22kernel}
Zhen Huang, Nabarun Deb, and Bodhisattva Sen.
\newblock Kernel partial correlation coefficient -- a measure of conditional
  dependence.
\newblock \emph{Journal of Machine Learning Research}, 23\penalty0
  (216):\penalty0 1--58, 2022.

\bibitem[Jiao and Vert(2016)]{jiao16kendall}
Yunlong Jiao and Jean-Philippe Vert.
\newblock The {K}endall and {M}allows kernels for permutations.
\newblock In \emph{International Conference on Machine Learning (ICML)}, pages
  2982--2990, 2016.

\bibitem[Jitkrittum et~al.(2016)Jitkrittum, Szab{\'o}, Chwialkowski, and
  Gretton]{jitkrittum16interpretable}
Wittawat Jitkrittum, Zolt{\'a}n Szab{\'o}, Kacper Chwialkowski, and Arthur
  Gretton.
\newblock Interpretable distribution features with maximum testing power.
\newblock In \emph{Advances in Neural Information Processing Systems (NIPS)},
  pages 181--189, 2016.

\bibitem[Jitkrittum et~al.(2017)Jitkrittum, Szab{\'o}, and
  Gretton]{jitkrittum17adaptive2}
Wittawat Jitkrittum, Zolt{\'a}n Szab{\'o}, and Arthur Gretton.
\newblock An adaptive test of independence with analytic kernel embeddings.
\newblock In \emph{International Conference on Machine Learning (ICML)}, pages
  1742--1751, 2017.

\bibitem[Kir{\'a}ly and Oberhauser(2019)]{kiraly19kernel}
Franz~J. Kir{\'a}ly and Harald Oberhauser.
\newblock Kernels for sequentially ordered data.
\newblock \emph{Journal of Machine Learning Research}, 20:\penalty0 1--45,
  2019.

\bibitem[Liu et~al.(2020)Liu, Xu, Liu, Zhang, Gretton, and
  Sutherland]{liu20learning}
Feng Liu, Wenkai Xu, Jie Liu, Guangquan Zhang, Arthur Gretton, and Danica~J.
  Sutherland.
\newblock Learning deep kernels for non-parametric two-sample tests.
\newblock In \emph{International Conference on Machine Learning (ICML)}, pages
  6316--6326, 2020.

\bibitem[Lodhi et~al.(2002)Lodhi, Saunders, Shawe-Taylor, Cristianini, and
  Watkins]{lodhi02text}
Huma Lodhi, Craig Saunders, John Shawe-Taylor, Nello Cristianini, and Chris
  Watkins.
\newblock Text classification using string kernels.
\newblock \emph{Journal of Machine Learning Research}, 2:\penalty0 419--444,
  2002.

\bibitem[Lyons(2013)]{lyons13distance}
Russell Lyons.
\newblock Distance covariance in metric spaces.
\newblock \emph{The Annals of Probability}, 41:\penalty0 3284--3305, 2013.

\bibitem[Micchelli et~al.(2006)Micchelli, Xu, and Zhang]{micchelli06universal}
Charles Micchelli, Yuesheng Xu, and Haizhang Zhang.
\newblock Universal kernels.
\newblock \emph{Journal of Machine Learning Research}, 7:\penalty0 2651--2667,
  2006.

\bibitem[Mooij et~al.(2016)Mooij, Peters, Janzing, Zscheischler, and
  Sch{\"o}lkopf]{mooij16distinguishing}
Joris Mooij, Jonas Peters, Dominik Janzing, Jakob Zscheischler, and Bernhard
  Sch{\"o}lkopf.
\newblock Distinguishing cause from effect using observational data: Methods
  and benchmarks.
\newblock \emph{Journal of Machine Learning Research}, 17:\penalty0 1--102,
  2016.

\bibitem[M{\"u}ller(1997)]{muller97integral}
Alfred M{\"u}ller.
\newblock Integral probability metrics and their generating classes of
  functions.
\newblock \emph{Advances in Applied Probability}, 29:\penalty0 429--443, 1997.

\bibitem[Pfister et~al.(2018)Pfister, B{\"u}hlmann, Sch{\"o}lkopf, and
  Peters]{pfister18kernel}
Niklas Pfister, Peter B{\"u}hlmann, Bernhard Sch{\"o}lkopf, and Jonas Peters.
\newblock Kernel-based tests for joint independence.
\newblock \emph{Journal of the Royal Statistical Society: Series B (Statistical
  Methodology)}, pages 5--31, 2018.

\bibitem[Quadrianto et~al.(2009)Quadrianto, Song, and
  Smola]{quadrianto09kernelized}
Novi Quadrianto, Le~Song, and Alex Smola.
\newblock Kernelized sorting.
\newblock In \emph{Advances in Neural Information Processing Systems (NIPS)},
  pages 1289--1296, 2009.

\bibitem[Rahimi and Recht(2007)]{rahimi07random}
Ali Rahimi and Benjamin Recht.
\newblock Random features for large-scale kernel machines.
\newblock In \emph{Advances in Neural Information Processing Systems (NIPS)},
  pages 1177--1184, 2007.

\bibitem[Sch{\"o}lkopf and Smola(2002)]{scholkopf02learning}
Bernhard Sch{\"o}lkopf and Alexander Smola.
\newblock \emph{Learning with Kernels: Support Vector Machines, Regularization,
  Optimization, and Beyond}.
\newblock MIT Press, 2002.

\bibitem[Sch{\"{o}}lkopf et~al.(2021)Sch{\"{o}}lkopf, Locatello, Bauer, Ke,
  Kalchbrenner, Goyal, and Bengio]{scholkopf21causal}
Bernhard Sch{\"{o}}lkopf, Francesco Locatello, Stefan Bauer, Nan~Rosemary Ke,
  Nal Kalchbrenner, Anirudh Goyal, and Yoshua Bengio.
\newblock Toward causal representation learning.
\newblock \emph{Proceedings of the {IEEE}}, 109\penalty0 (5):\penalty0
  612--634, 2021.

\bibitem[Schrab et~al.(2022)Schrab, Kim, Guedj, and
  Gretton]{schrab22incompleteu}
Antonin Schrab, Ilmun Kim, Benjamin Guedj, and Arthur Gretton.
\newblock Efficient aggregated kernel tests using incomplete {U}-statistics.
\newblock In \emph{Advances in Neural Information Processing Systems ({NIPS})},
  pages 18793--18807, 2022.

\bibitem[Sejdinovic et~al.(2013{\natexlab{a}})Sejdinovic, Gretton, and
  Bergsma]{sejdinovic13kernel}
Dino Sejdinovic, Arthur Gretton, and Wicher Bergsma.
\newblock A kernel test for three-variable interactions.
\newblock In \emph{Advances in Neural Information Processing Systems (NIPS)},
  pages 1124--1132, 2013{\natexlab{a}}.

\bibitem[Sejdinovic et~al.(2013{\natexlab{b}})Sejdinovic, Sriperumbudur,
  Gretton, and Fukumizu]{sejdinovic13equivalence}
Dino Sejdinovic, Bharath Sriperumbudur, Arthur Gretton, and Kenji Fukumizu.
\newblock Equivalence of distance-based and {RKHS}-based statistics in
  hypothesis testing.
\newblock \emph{Annals of Statistics}, 41:\penalty0 2263--2291,
  2013{\natexlab{b}}.

\bibitem[Serfling(1980)]{serfling80approximation}
Robert~J. Serfling.
\newblock \emph{Approximation theorems of mathematical statistics}.
\newblock John Wiley \& Sons, 1980.

\bibitem[Sheng and Sriperumbudur(2023)]{sheng23distance}
Tianhong Sheng and Bharath~K. Sriperumbudur.
\newblock On distance and kernel measures of conditional independence.
\newblock \emph{Journal of Machine Learning Research}, 24\penalty0
  (7):\penalty0 1--16, 2023.

\bibitem[Smola et~al.(2007)Smola, Gretton, Song, and
  Sch{\"o}lkopf]{smola07hilbert}
Alexander Smola, Arthur Gretton, Le~Song, and Bernhard Sch{\"o}lkopf.
\newblock A {H}ilbert space embedding for distributions.
\newblock In \emph{Algorithmic Learning Theory (ALT)}, pages 13--31, 2007.

\bibitem[Song et~al.(2007)Song, Smola, Gretton, and
  Borgwardt]{song07dependence}
Le~Song, Alexander~J. Smola, Arthur Gretton, and Karsten~M. Borgwardt.
\newblock A dependence maximization view of clustering.
\newblock In \emph{International Conference on Machine Learning (ICML)}, pages
  815–--822, 2007.

\bibitem[Song et~al.(2012)Song, Smola, Gretton, Bedo, and
  Borgwardt]{song12feature}
Le~Song, Alex Smola, Arthur Gretton, Justin Bedo, and Karsten Borgwardt.
\newblock Feature selection via dependence maximization.
\newblock \emph{Journal of Machine Learning Research}, 13\penalty0
  (1):\penalty0 1393--1434, 2012.

\bibitem[Sriperumbudur et~al.(2010)Sriperumbudur, Gretton, Fukumizu,
  Sch{\"o}lkopf, and Lanckriet]{sriperumbudur10hilbert}
Bharath Sriperumbudur, Arthur Gretton, Kenji Fukumizu, Bernhard Sch{\"o}lkopf,
  and Gert Lanckriet.
\newblock Hilbert space embeddings and metrics on probability measures.
\newblock \emph{Journal of Machine Learning Research}, 11:\penalty0 1517--1561,
  2010.

\bibitem[Steinwart(2001)]{steinwart01influence}
Ingo Steinwart.
\newblock On the influence of the kernel on the consistency of support vector
  machines.
\newblock \emph{Journal of Machine Learning Research}, 2:\penalty0 67--93,
  2001.

\bibitem[Steinwart and Christmann(2008)]{steinwart08support}
Ingo Steinwart and Andreas Christmann.
\newblock \emph{Support Vector Machines}.
\newblock Springer, 2008.

\bibitem[Szab{\'o} and Sriperumbudur(2018)]{szabo18characteristic2}
Zolt{\'a}n Szab{\'o} and Bharath~K. Sriperumbudur.
\newblock Characteristic and universal tensor product kernels.
\newblock \emph{Journal of Machine Learning Research}, 18\penalty0
  (233):\penalty0 1--29, 2018.

\bibitem[Sz{\'e}kely and Rizzo(2009)]{szekely09brownian}
G{\'a}bor~J. Sz{\'e}kely and Maria~L. Rizzo.
\newblock Brownian distance covariance.
\newblock \emph{The Annals of Applied Statistics}, 3:\penalty0 1236--1265,
  2009.

\bibitem[Sz{\'e}kely et~al.(2007)Sz{\'e}kely, Rizzo, and
  Bakirov]{szekely07measuring}
G{\'a}bor~J. Sz{\'e}kely, Maria~L. Rizzo, and Nail~K. Bakirov.
\newblock Measuring and testing dependence by correlation of distances.
\newblock \emph{The Annals of Statistics}, 35:\penalty0 2769--2794, 2007.

\bibitem[Wang et~al.(2022)Wang, Du, Zhang, and Shi]{wang22rank}
Andi Wang, Juan Du, Xi~Zhang, and Jianjun Shi.
\newblock Ranking features to promote diversity: An approach based on sparse
  distance correlation.
\newblock \emph{Technometrics}, 64\penalty0 (3):\penalty0 384--395, 2022.

\bibitem[Watkins(1999)]{watkins99dynamic}
Chris Watkins.
\newblock Dynamic alignment kernels.
\newblock In \emph{Advances in Neural Information Processing Systems (NIPS)},
  pages 39--50, 1999.

\bibitem[Willard(1970)]{willard70general}
Stephen Willard.
\newblock \emph{General Topology}.
\newblock Addison-Wesley, 1970.

\bibitem[Williams and Seeger(2001)]{williams01using}
Christopher Williams and Matthias Seeger.
\newblock Using the {N}ystr{\"{o}}m method to speed up kernel machines.
\newblock In \emph{Advances in Neural Information Processing Systems (NIPS)},
  pages 682--688, 2001.

\bibitem[Zhang et~al.(2018)Zhang, Filippi, Gretton, and
  Sejdinovic]{zhang18large}
Qinyi Zhang, Sarah Filippi, Arthur Gretton, and Dino Sejdinovic.
\newblock Large-scale kernel methods for independence testing.
\newblock \emph{Statistics and Computing}, 28\penalty0 (1):\penalty0 1--18,
  2018.

\bibitem[Zolotarev(1983)]{zolotarev83probability}
V.~Zolotarev.
\newblock Probability metrics.
\newblock \emph{Theory of Probability and its Applications}, 28:\penalty0
  278--302, 1983.

\end{thebibliography}

\end{document}